\documentclass{article}
\usepackage{authblk}

\usepackage{chngcntr}

\usepackage{natbib}
\usepackage[utf8]{inputenc}
\usepackage{amsmath,amssymb}
\usepackage{graphicx}
\usepackage[utf8]{inputenc}
\usepackage[english]{babel}
\usepackage[toc,page]{appendix}
\usepackage[colorlinks=TRUE, linkcolor=blue, urlcolor=blue, citecolor=blue]{hyperref}
\usepackage[paper=portrait,pagesize]{typearea}
\graphicspath{ {./images/} }
\usepackage[margin=25mm]{geometry}
\usepackage[nocenter]{qtree}
\usepackage{tikz}
\usetikzlibrary{arrows}
\usepackage{breqn}
\usepackage{multirow}
\usepackage{floatrow}
\usepackage{amsthm,bm}
\usepackage[belowskip=-9pt,aboveskip=0pt]{caption}
\usepackage{graphicx}
\usepackage{float}
\usepackage[toc,page]{appendix}
\usepackage{tikz}
\usepackage[capitalise]{cleveref}
\usepackage{booktabs}
\usepackage{enumitem}

\usetikzlibrary{positioning, arrows.meta}
\usepackage{float}
\floatstyle{plaintop}
\restylefloat{table}
\usepackage{graphicx}
\usepackage{subcaption}
\newtheorem{theorem}{Theorem}[section]

\usepackage{amsthm}
\usepackage{booktabs}
\usepackage{threeparttable}

\newtheorem{proposition}[theorem]{Proposition}
\theoremstyle{remark}

\usetikzlibrary{bayesnet}

\tikzset{
  treenode/.style = {align=center, inner sep=0pt, text centered,
    font=\sffamily},
  arn_n/.style = {treenode, circle, black, font=\sffamily\bfseries, draw=black,
    fill=white, text width=1.5em}
}

\title{Structural Concentration in Weighted Networks: A Class of Topology-Aware Indices.} 

\author[$\star$]{L. Riso}
\author[$\star$]{M.G. Zoia\thanks{Corresponding author. Email: maria.zoia@unicatt.it } }

\affil[$\star$]{Department of Economic Policy, Università Cattolica del Sacro Cuore, Milan, Italy}

\begin{document}

\maketitle

\begin{abstract}
This paper develops a unified framework for measuring concentration in weighted systems embedded in networks of interactions. While traditional indices such as the Herfindahl–Hirschman Index capture dispersion in weights, they neglect the topology of relationships among the elements receiving those weights. To address this limitation, we introduce a family of topology-aware concentration indices that jointly account for weight distributions and network structure.

At the core of the framework lies a baseline Network Concentration Index (NCI), defined as a normalized quadratic form that measures the fraction of potential weighted interconnection realized along observed network links. Building on this foundation, we construct a flexible class of extensions that modify either the interaction structure or the normalization benchmark, including weighted, density-adjusted, null-model, degree-constrained, transformed-data, and multi-layer variants.

This family of indices preserves key properties such as normalization, invariance, and interpretability, while allowing concentration to be evaluated across different dimensions of dependence, including intensity, higher-order interactions, and extreme events. Theoretical results characterize the indices and establish their relationship with classical concentration and network measures.

Empirical and simulation evidence demonstrate that systems with identical weight distributions may exhibit markedly different levels of structural concentration depending on network topology, highlighting the additional information captured by the proposed framework. The approach is broadly applicable to economic, financial, and complex systems in which weighted elements interact through networks.

\end{abstract}
\textbf{Keywords} Weighted Networks; Network Concentration; Network Topology; Structural Dependence; Systemic Risk. \\
\textbf{JEL codes} C02,C65, D85, C63.
\section{Introduction}

Concentration is a fundamental characteristic of weighted systems and
plays an important role in the analysis of economic, financial, social,
biological, and technological structures \citep{hirschman1958strategy,wagstaff1991measurement,gini1912variability}. Whenever a set of elements is
assigned heterogeneous weights, the degree to which those weights are
concentrated or dispersed affects the resilience, exposure, and
functional organization of the system. Highly concentrated weight
distributions tend to increase vulnerability to shocks affecting a small
subset of important elements, whereas more dispersed distributions may
provide greater balance and robustness. For this reason, concentration
measures play a central role in many fields, including industrial
organization, network science, risk analysis, and resource allocation.

The most widely used measures of concentration are based solely on the
distribution of weights. Classical examples include the
Herfindahl--Hirschman Index (HHI) \citep{hirschman1958strategy}, originally introduced in the study of
market concentration in industrial organization ,
, as well as other inequality and dispersion measures
such as the \cite{gini1912variability} coefficient  and entropy-based indices
proposed by \cite{fisher1970economics}. These measures summarize the dispersion of weights
across elements and provide simple and interpretable indicators of
concentration. Because of their analytical simplicity and clear
interpretation, they have been widely applied in economics, finance, and
regulatory analysis \citep[among many others]{rhoades1993herfindahl,cowell2011measuring,atkinson1970measurement}.

Despite their usefulness, traditional concentration measures evaluate
concentration purely in weight space and therefore ignore an essential
dimension of many real-world systems: the structure of relationships
among the elements receiving those weights. Two systems with identical
weight distributions may exhibit substantially different levels of
effective concentration if their underlying interdependence structures
differ. A system whose largest weights are assigned to strongly
interconnected elements may be much more structurally concentrated than
another system with the same weight dispersion but weaker or more
fragmented connectivity.

This limitation becomes increasingly important in settings where the
behaviour of the system depends not only on the distribution of weights
but also on the topology of interactions among its components. In
economic and financial applications, for example, exposures and
allocations may be distributed across firms, sectors, institutions, or
assets whose relationships are shaped by dependence, correlation, or
transmission mechanisms. More generally, in any weighted system embedded
in a network of interactions, the effective degree of concentration
depends not only on how weight is distributed across nodes but also on
how those nodes are interconnected.

Network science provides powerful tools for
characterizing interdependence structures using graph-theoretic
representations \citep{brede2012networks, jackson2008social, posfai2016network}. In such 
representations, nodes denote the elements of the system and edges
represent interactions such as dependence, similarity, or transmission
channels \citep[among many others]{riso2022concept,jordan2004graphical,hojsgaard2012graphical}. Network methods have been widely applied in economics and
finance to analyse systems of interconnected agents, including
production networks \citep{acemoglu2012network, carvalho2014micro}, financial
contagion and systemic risk \citep{allen2000financial, elliott2014financial}, and
dependence structures in financial markets \citep{mantegna1999hierarchical}. These
approaches provide powerful tools for identifying clusters, central
nodes, and patterns of connectivity within complex systems \citep{pozzi2008centrality}.

However, existing network measures typically characterize the topology
of the interaction network independently of the distribution of weights
across nodes. As a result, they do not provide a direct measure of the
structural concentration associated with a particular allocation of
weights within the network. Consequently, traditional concentration
measures ignore network structure, while network statistics ignore the
distribution of weights.

This paper proposes a unified framework for measuring concentration in weighted networks, in which a family of topology-aware concentration indices is constructed by combining weight distributions with interaction structures. Rather than introducing a single measure, the paper develops a class of indices that differ in their specification of the interaction functional and the normalization benchmark, thereby allowing structural concentration to be evaluated along multiple dimensions of dependence.

At the core of this framework lies a baseline specification, the Network Concentration Index (NCI), which measures the proportion of potential weighted interconnection that is realized along direct network links. The index is defined as a normalized quadratic form involving the weight vector and the adjacency matrix of a filtered network, and can be interpreted as a structural extension of classical concentration measures.

Building on this baseline formulation, the framework admits several extensions. Alternative specifications incorporate link intensity through weighted interaction matrices, isolate specific dependence features through transformed-data networks, or aggregate multiple interaction channels in a multi-layer setting. In addition, different normalization schemes allow concentration to be evaluated relative to density-adjusted, stochastic, or degree-constrained benchmarks. Taken together, these constructions define a flexible and internally consistent family of concentration indices that generalizes traditional weight-based measures while preserving normalization and invariance properties.

The proposed framework is intentionally general and can be applied in a
wide range of contexts whenever weighted elements are embedded in a
network of interactions. These contexts include industrial and market
structures, production and supply networks, financial and interbank
systems, asset allocation problems, and social or technological
networks. In all these cases, the interaction between weight
heterogeneity and network topology plays a central role in determining
the effective structure of the system.
The NCI also connects naturally to 
several strands of the Operations Research literature \citep{inomata2024measuring,gomez2013modeling,ouraga2021principal,parsa2018analysis}. 
In portfolio optimisation, concentration measures play a 
central role as constraints or objective components in 
mean-variance and risk-parity models \citep{scherer2007portfolio,maillard2010properties,roncalli2016risk}: the NCI extends this 
toolbox by providing a topology-aware diversification 
diagnostic that penalises not only weight concentration 
but also the structural clustering of large positions 
within correlated subgraphs. In supply chain and network 
design, concentration of flows or capacities along 
critical links is a primary driver of systemic 
vulnerability \citep{snyder2016or,hendricks2005empirical}: the NCI provides a 
scalar summary of this vulnerability that is comparable 
across network topologies and weight distributions. 
More broadly, the index can be embedded as an objective 
or constraint in network optimisation problems, 
for instance to find weight allocations that minimise 
structural concentration subject to performance 
constraints, or to identify the network topology that 
maximises resilience for a given weight distribution. 
These connections motivate the development of a 
theoretically grounded, computationally tractable 
concentration measure that integrates weight 
heterogeneity with network structure.
The main contributions of the paper are threefold. First, we introduce a unified framework for measuring concentration that integrates weight distributions with network structure. Second, we show that the baseline specification — the Network Concentration Index — extends classical quadratic concentration measures in a natural and coherent way. Third, we develop a family of extensions that allow structural concentration to be measured across different dependence representations, including baseline dependence, higher-order dependence, and extreme-event dependence.


The remainder of the paper is organized as follows. 
Section~\ref{sec:classical_concentration} reviews the main approaches to concentration measurement 
and discusses their limitations in the presence of network 
interdependence. Section~\ref{sec:NCI_definition} introduces the Network Concentration 
Index and its extensions, and establishes their main analytical 
properties. Section~\ref{sec:simulation} investigates the behaviour of the proposed 
indices through simulation and Monte Carlo experiments, illustrating 
how they respond to different network topologies and weight 
distributions. Section~\ref{sec:empirical} presents empirical applications to 
production and trade networks derived from the World Input–Output 
Database, while Section~\ref{sec:empirical_finance} illustrates the approach using a financial 
dependence network constructed from equity market data. 
Section~\ref{sec:conclusions} concludes.

\section{Concentration Measures and Network Structure}
\label{sec:classical_concentration}

The measurement of concentration plays a central role in economics,
statistics, and network science \citep[among many others]{kvaalseth2018relationship,lorenz1905methods,atkinson1970measurement,jackson2008social}. Concentration indices are widely used
to quantify how unevenly a quantity such as market share, activity,
exposure, or resources is distributed across a set of elements.
Traditional approaches focus on the dispersion of weights, whereas more
recent work studies systems in which weighted elements interact through
 network structures.
 This section reviews the main strands of the literature and highlights the gap addressed by the proposed framework of topology-aware concentration indices.

\subsection{Weight-Based Concentration Measures}

Traditional measures of concentration evaluate the dispersion of a
weight distribution. Let $w=(w_1,\dots,w_N)'$ denote a vector of
nonnegative weights satisfying $\sum_{i=1}^N w_i = 1$. A widely used
indicator is the Herfindahl--Hirschman Index (HHI)
\citep{hirschman1958strategy,herfindahl1997concentration},
\begin{gather}
HHI = \sum_{i=1}^{N} w_i^2 
\end{gather}
which takes values in the interval $1/N \le HHI \le 1$, with the lower
bound corresponding to equal weights and the upper bound to complete
concentration.

Other measures of dispersion include entropy-based indices such as the
Theil index \citep{theil1973new} and inequality measures such as
the Gini coefficient \citep{gini1912variability} \citep[see also][]{sen1997economic}.

While widely used, these measures evaluate concentration purely in
weight space and therefore ignore the relationships among the elements
receiving those weights. Consequently, systems with identical weight
distributions may exhibit different structural properties when the
elements are interconnected through dependence or interaction networks.
\subsection{Contribution-Based Measures}

A second strand of the literature evaluates concentration through the
contribution of individual elements to an aggregate quantity. Let
$F(w)$ denote a homogeneous aggregate function of the weights, such as
output, activity, or risk. The marginal contribution of element $i$ is
$\partial F/\partial w_i$, and the corresponding total contribution is
$C_i = w_i \, \partial F/\partial w_i$. Under standard regularity
conditions, Euler's theorem implies the decomposition
\begin{gather}
F(w) = \sum_{i=1}^N C_i 
\end{gather}

Contribution-based decompositions are widely used in applications such
as portfolio risk allocation and capital attribution
\citep{litterman1996hot,tasche2007capital}. However, these measures
depend on the specific functional form of the aggregate function $F$
and therefore do not provide a model-independent measure of structural
relationships among elements.
\subsection{Network-Based Measures of Structure}

When the elements of a system interact with one another, their
relationships can naturally be represented using networks \citep{jordan2004graphical}. Let

\begin{equation*}
G=(V,E)    
\end{equation*}
denote a network (or graph $G$) in which nodes (vertices $V$) represent elements of the system and
edges (or links $E$) represent interactions among them \citep{jordan2004graphical}. The network can be represented
by an adjacency matrix:
\begin{equation*}
A \in \{0,1\}^{N\times N}
\end{equation*}
which define the relationships between the nodes through  the corresponding links \citep{riso2022concept}.
Network science provides a wide range of tools for characterizing the
topology of interaction structures. Examples include node degree,
centrality measures, clustering coefficients, community detection, and
network density \citep{newman2010networks, jackson2008social, posfai2016network}.

These methods have been widely applied in economics and finance to
analyze systems of interacting agents. Important examples include
production networks and the propagation of sectoral shocks
\citep{acemoglu2012network, carvalho2014micro}, financial contagion in interbank
networks \citep{allen2000financial, elliott2014financial}, and dependence networks in
financial markets constructed from filtered correlation matrices
\citep{mantegna1999hierarchical}.

While these approaches provide powerful tools for describing the
topology of interaction structures, standard network statistics
typically characterize network structure independently of the
distribution of weights across nodes. As a result, they do not provide
a direct measure of concentration associated with a particular
allocation of weights within a network.

The literature reviewed above highlights three complementary
perspectives on concentration.
None of these approaches provides a measure that simultaneously
accounts for both the distribution of weights and the topology of
interactions among the elements of a system.
%
For instance, centrality-based measures (e.g., eigenvector or Katz centrality \cite{katz1953status}) assign importance scores to individual nodes based on their position in the network, but do not provide a scalar measure of concentration for a given weight distribution. Systemic risk measures and contagion-based metrics depend on specific propagation mechanisms and therefore require additional modelling assumptions. More generally, quadratic interaction and exposure measures summarize dependence across node pairs but typically do not isolate the role of network topology relative to weight dispersion. In contrast, the family of the Network Concentration Index evaluates the alignment between economically relevant weights and observed network links, providing a distribution-level measure of structural concentration that is directly comparable across different network configurations.

This family is based on quadratic interaction measures that appear in several areas
of economics, finance, and spatial statistics. For example, quadratic interaction terms arise in models of production
networks \citep{acemoglu2012network} and in models of financial contagion and
interconnected balance sheets \citep{allen2000financial, elliott2014financial}.
They
also appear in spatial statistics
through spatial autocorrelation measures such as \cite{moran1950notes}'s $I$. Moran's statistic evaluates whether similar values
of a variable tend to occur in neighbouring locations within a spatial
network.

More generally, quadratic exposure measures summarize the level of
interaction induced by a given network structure but typically do not
provide a normalization that separates the structural component of
interaction from the dispersion of the underlying weight vector.
By contrast, the family of indices proposed in this paper evaluates
the extent to which economically important nodes (as measured by their
weights) are interconnected within a network, under alternative
specifications of both the interaction structure and the normalization
benchmark.

Accordingly, these indices measure realized weighted interaction
relative to a benchmark level of potential interaction determined by
the chosen normalization. In this sense, the proposed framework extends
classical quadratic concentration measures to settings in which
weighted elements are embedded in interaction networks, thereby
capturing a structural dimension of concentration that cannot be
detected using traditional weight-based metrics.

The following section introduces the proposed framework — and its baseline specification, the Network Concentration Index (NCI) — together with its main analytical properties.
\section{A Unified Framework of Network Concentration Indices}
\label{sec:NCI_definition}

This section introduces a unified framework for topology-aware concentration measures. The proposed indices can be written as ratios of quadratic forms, where the numerator captures realized weighted interaction under a given network structure, and the denominator provides a benchmark level of potential interaction. Within this general class, the baseline Network Concentration Index (NCI) arises as a canonical specification. Alternative indices are then obtained as systematic extensions through modifications of either the interaction matrix or the benchmark normalization.

\subsection{Unified formulation}

Let $w \in \mathbb{R}^n$ be a vector of non-negative weights such that $\sum_i w_i = 1$. 
Depending on the application, the weights may be indexed by time or treated as static.

We define the general class of network concentration indices specified as follows
\begin{equation}
\Psi(w;M,B) = \frac{w^\top M w}{w^\top B w}
\label{eq:master_quadratic_nci}
\end{equation}

The matrix $M$ encodes the structure along which weighted interaction is realized, while $B$ determines the normalization relative to a benchmark configuration. Different choices of $(M,B)$ generate different indices within a unified representation.

To ensure that $\Psi(w;M,B)$ is well-defined and interpretable, we impose the following conditions:
\begin{itemize}
    \item $M$ and $B$ are symmetric;
    \item $w^\top B w > 0$ for all admissible $w$;
    \item $B$ represents a meaningful benchmark of potential interactions.
\end{itemize}

\subsection{The baseline index as a canonical specification}

The baseline Network Concentration Index is obtained by choosing:
\begin{equation}
M = A, \qquad B = \mathbf{1}\mathbf{1}^\top - I,
\end{equation}
where $A$ is the (binary) adjacency matrix of the observed network, $\mathbf{1}$ is the vector of ones, and $I$ is the identity matrix.

Substituting into the unified formulation yields
\begin{equation}
\psi(w,A) = \frac{w^\top A w}{w^\top (\mathbf{1}\mathbf{1}^\top - I) w}.
\end{equation}

Using $\sum_i w_i = 1$, the denominator simplifies to
\begin{equation}
w^\top (\mathbf{1}\mathbf{1}^\top - I) w = 1 - \sum_i w_i^2,
\end{equation}
so that
\begin{equation}\label{eq:nci_baseline}
\psi(w,A) = \frac{w^\top A w}{1 - \sum_i w_i^2}.
\end{equation}

The index numerator, say $\Theta=w^{\top} A w$, measures the total weighted interaction
allocated along the edges of the network: it increases when large
weights are assigned to nodes that are directly connected in the
interaction structure. Quadratic forms of this type are common in the
analysis of networked systems and dependence structures, where they
summarize the second-order interaction effects implied by a weighted
configuration \citep{acemoglu2012network,elliott2014financial}.

Regarding the denominator, since $\sum_i w_{i}=1$, it follows that:
\begin{equation}
\label{eq:Lambda_simplify}
\Lambda
=
\left(\sum_{i=1}^{N} w_{i}\right)^2 - \sum_{i=1}^{N} w_{i}^2
=
1-\sum_{i=1}^{N} w_{i}^2=1 - HHI
\end{equation}
where $HHI=\sum_i w_{i}^2$ corresponds to the 
Hirschman index
\citep{hirschman1958strategy,herfindahl1997concentration}.

The normalization term $\Lambda = 1-HHI$ represents the total
pairwise interaction implied by the weight distribution alone. It
corresponds to the maximum weighted interconnection that could occur
if all pairs of nodes were linked, and therefore provides a benchmark
that depends only on weight dispersion and not on network topology.

In light of the above developments, the Network Concentration Index can be represented as
\begin{equation}
\label{eq:NCI_def}
\psi
=
\frac{\Theta}{\Lambda}
=
\frac{w^{\top} Aw}{1-HHI}
\qquad
\text{for } \Lambda>0
\end{equation}
When $\Lambda=0$---which occurs only when the entire weight is assigned to a single node---the index $\psi$ is not informative, because no pairwise interaction can be formed.
in empirical applications, this degenerate
case can be excluded by construction or handled by setting
$\psi=0$ by convention.

The NCI is a topology-aware concentration measure. The numerator
$\Theta$ captures the structural interaction realized by measuring
how much total weight is placed on pairs of nodes that are directly
linked in the observed network. The denominator $\Lambda$ captures the
potential interconnection induced by weights alone, that is, the
maximum pairwise interaction compatible with the weight distribution
under a complete network. Thus $\Lambda$ is increasing in weight
dispersion and decreasing in weight concentration.

Accordingly, $\psi$ measures the fraction of potential weighted
interconnection that is effectively realized along direct network links:
\begin{equation}
    \psi
    =
    \frac{\text{realized interaction (along observed links)}}
         {\text{maximum interaction implied by weight dispersion}}
    \label{eq:Psi}
\end{equation}
For a symmetric binary adjacency matrix with zero diagonal, $\psi\in [0,1]$ by construction.
Values approaching zero arise
when large weights are assigned to nodes sharing weak or absent
connections, so that the realized quadratic form $w^{\top} A w$
remains small relative to its upper bound. Conversely, values approaching
unity indicate that heavy weights are concentrated on strongly connected
node pairs, so that the network structure amplifies rather than
attenuates weight concentration. $\psi$ can therefore be interpreted
as a network-adjusted concentration measure: by normalizing
$w^{\top} A w$ by the maximum interaction consistent with the
observed weight distribution, the index removes the mechanical
contribution of weight dispersion and isolates the incremental role of
network topology.

The NCI enjoys several properties that make it suitable for
applied analysis. It is bounded in $[0,1]$ under binary adjacency,
invariant to node relabeling, and admits transparent limiting cases.
By normalizing with respect to $1 - \mathrm{HHI}$, it disentangles
the purely mechanical effect of weight concentration from the structural
effect of allocating weight to interconnected nodes. This decomposition
makes the NCI a diagnostic tool for comparing the interaction structures
activated by different weighted configurations over the same network, or
by the same weight distribution over different network topologies.


It is also instructive to contrast the NCI with covariance-based
interaction measures. Quadratic forms defined on covariance or
dependence matrices summarize aggregate interaction in moment space,
capturing the full pairwise dependence structure. The NCI, by contrast,
evaluates concentration in network space, restricting attention to
direct links in the interaction graph. The two approaches are therefore
complementary: covariance-based measures characterize interaction
through the complete dependence matrix, while the NCI focuses on the
topology of observed connections. In settings where both the weight
distribution and the network structure are relevant---such as financial
contagion, supply-chain risk, or ecological flow networks---using the
two measures in conjunction provides a richer characterization of
systemic concentration than either alone.


The Network Concentration Index possesses a set of analytical 
properties that clarify its interpretation and justify its 
construction. We present these results through four propositions. 
Proposition~\ref{prop:wavg} establishes that the index admits a 
weighted-average representation over pairwise interaction 
intensities. Proposition~\ref{prop:norm} shows that the 
normalization adopted is the unique denominator consistent with 
a complete-network benchmark.  
These first two propositions concern the subclass of indices that retain the complete-network normalization, i.e. $B=\mathbf{1}\mathbf{1}^\top-I$.
Proposition~\ref{prop:assort} 
connects the index to the concept of assortative connectivity. 
Finally, Proposition~\ref{prop:random} provides a random-network 
benchmark that links the expected value of the index to network 
density. 

In the following, $w = (w_1, \dots, w_N)^{\top}$ denotes 
a non-negative weight vector satisfying $\sum_{i=1}^{N} w_i = 1$, 
and $M = [M_{ij}]$ denotes a symmetric matrix with zero diagonal 
and non-negative off-diagonal entries.

\begin{proposition}[\textbf{Weighted-average representation}]
\label{prop:wavg}

Let $M$ and $w$ satisfy the conditions above, and suppose 
$1 - \sum_{i=1}^{N} w_i^2 > 0$. Then
\begin{gather}
    \psi^{(M)}(w)
    =
    \frac{w^{\top} M w}{1 - \sum_{i=1}^{N} w_i^2}
    =
    \frac{\displaystyle\sum_{i \neq j} w_i w_j M_{ij}}
         {\displaystyle\sum_{i \neq j} w_i w_j}
\end{gather}
that is, $\psi^{(M)}(w)$ equals the weighted average of 
$M_{ij}$ across all ordered node pairs $(i,j)$,
with each 
pair weighted proportionally to $w_i w_j$.
\end{proposition}

\begin{proof}
Since $M_{ii} = 0$ for all $i$, we have
\begin{gather}
    w^{\top} M w = \sum_{i \neq j} w_i w_j M_{ij}
\end{gather}
Moreover, using $\sum_i w_i = 1$,
\begin{gather}
    1 - \sum_{i=1}^{N} w_i^2
    =
    \Bigl(\sum_{i=1}^{N} w_i\Bigr)^2 - \sum_{i=1}^{N} w_i^2
    =
    \sum_{i \neq j} w_i w_j
\end{gather}
Substituting both expressions into the definition of 
$\psi^{(M)}(w)$ yields the result.
\end{proof}

\noindent
Proposition~\ref{prop:wavg} provides the key interpretive 
result: the NCI aggregates pairwise interaction intensities 
through a weighting scheme that reflects the joint importance 
of each node pair in the weight distribution. The next 
proposition shows that the denominator $1 - \sum_i w_i^2$ 
is not merely convenient, but is in fact the unique 
normalization consistent with a natural benchmark condition.
A generalization of this result to the full unified family 
of indices is provided in 
Appendix~\ref{sec:appendix} 
(Proposition~\ref{app:prop:wa}).

\begin{proposition}[\textbf{Normalization characterization}]
\label{prop:norm}

Consider indices of the form
\begin{gather}
    \Psi(w, M) = \frac{w^{\top} M w}{D(w)}
\end{gather}
where $D(w) > 0$ whenever at least two components of $w$ are 
positive. If
\begin{gather}
    \Psi\bigl(w,\, \mathbf{1}\mathbf{1}^{\top} - I\bigr) = 1
    \quad \text{for every admissible } w,
\end{gather}
then necessarily
\begin{gather}
    D(w)
    =
    w^{\top}(\mathbf{1}\mathbf{1}^{\top} - I)\,w
    =
    1 - \sum_{i=1}^{N} w_i^2
\end{gather}
\end{proposition}

\begin{proof}
Evaluating the complete-network normalization condition directly,
\begin{gather}
    1
    =
    \Psi\bigl(w,\, \mathbf{1}\mathbf{1}^{\top} - I\bigr)
    =
    \frac{w^{\top}(\mathbf{1}\mathbf{1}^{\top} - I)\,w}{D(w)}
\end{gather}
so $D(w) = w^{\top}(\mathbf{1}\mathbf{1}^{\top} - I)\,w$. Since 
$\sum_i w_i = 1$,
\begin{gather}
    w^{\top}(\mathbf{1}\mathbf{1}^{\top} - I)\,w
    =
    \Bigl(\sum_i w_i\Bigr)^2 - \sum_i w_i^2
    =
    1 - \sum_{i=1}^{N} w_i^2 
\end{gather}
\end{proof}

\noindent
The complete-network condition is natural: if all nodes are 
mutually connected, the fraction of realized interaction equals 
one regardless of the weight distribution. 
Proposition~\ref{prop:norm} shows that this single requirement 
uniquely pins down the denominator to $1 - HHI$, 
thereby giving the normalization an axiomatic foundation.
The explicit relationship between the normalization 
denominator and the Herfindahl--Hirschman Index is 
formalized in Appendix~\ref{sec:appendix} 
(Proposition~\ref{app:prop:hhi}).

\begin{proposition}[\textbf{Assortative connectivity}]
\label{prop:assort}

Let $A = [A_{ij}]$ be a symmetric binary adjacency matrix with 
zero diagonal, with $\sum_{i \neq j} A_{ij} > 0$. Then
\begin{gather}
    \psi(w, A)
    =
    \delta(A)\,
    \frac{\mathbb{E}[\,w_i w_j \mid A_{ij} = 1\,]}
         {\mathbb{E}[\,w_i w_j\,]},
\end{gather}
where
\begin{equation}
    \delta(A)
    =
    \frac{\sum_{i \neq j} A_{ij}}{N(N-1)}
\label{densit}
\end{equation}
    
is network density, $\mathbb{E}[\,w_i w_j \mid A_{ij} = 1\,]$ 
is the average pairwise weight product over linked pairs, and 
$\mathbb{E}[\,w_i w_j\,]$ is the average over all ordered 
pairs.
\end{proposition}

\begin{proof}
Starting from the weighted-average representation established 
in Proposition~\ref{prop:wavg}, multiply and divide by 
$\sum_{i \neq j} A_{ij}$ and by $N(N-1)$:
\begin{gather}
    \psi(w, A)
    =
    \frac{\sum_{i \neq j} w_i w_j A_{ij}}{\sum_{i \neq j} A_{ij}}
    \cdot
    \frac{\sum_{i \neq j} A_{ij}}{N(N-1)}
    \cdot
    \frac{N(N-1)}{\sum_{i \neq j} w_i w_j}
\end{gather}
Rearranging gives
\begin{gather}
    \psi(w, A)
    =
    \underbrace{
        \frac{\sum_{i \neq j} A_{ij}}{N(N-1)}
    }_{\delta(A)}
    \cdot
    \frac{
        \dfrac{\sum_{i \neq j} w_i w_j A_{ij}}
              {\sum_{i \neq j} A_{ij}}
    }{
        \dfrac{\sum_{i \neq j} w_i w_j}
              {N(N-1)}
    }
\end{gather}
which is the stated representation.
\end{proof}

\noindent
Proposition~\ref{prop:assort} provides an intuitive decomposition of the
Network Concentration Index (NCI). Specifically, the index can be written
as the product of two components. The first component is the network
density $\delta(A)$, which measures the overall connectivity of the graph
by capturing the fraction of possible links that are actually present.
The second component is the ratio between the average pairwise product of
node weights among connected pairs and the corresponding average across
all node pairs. This ratio measures the extent to which weights are aligned with that connectivity.
 When this ratio
exceeds one, the network displays assortative connectivity with respect
 to node weights, meaning that nodes associated with high weights tend to be connected with each other.
Consequently, the NCI is high when the network
is both densely connected and structurally assortative with respect to
node weights. 

The following proposition provides a tractable benchmark by
evaluating the expected value of the index under a canonical
random-network model. The canonical random-network model refers to the Erdős–Rényi
random graph, in which each pair of nodes is connected
independently with probability $p$. 
A formal derivation of the density-adjusted counterpart 
of this result, which expresses the assortative-connectivity 
ratio as a stand-alone index, is given in 
Appendix~\ref{sec:appendix} 
(Proposition~\ref{app:prop:dens}).

\begin{proposition}[\textbf{Random-network benchmark}]
\label{prop:random}
Let $A$ be the adjacency matrix of an Erd\H{o}s--R\'enyi random graph \citep{erdHos1960evolution}
in which each pair $(i,j)$ with $i \neq j$ is 
connected independently with probability $p \in (0,1)$. Then
\begin{gather}
    \mathbb{E}[\,\psi(w, A)\,] = p
\end{gather}
\end{proposition}

\begin{proof}
Since each $A_{ij}$ is a Bernoulli$(p)$ random variable 
independent of $w$,
\begin{gather}
    \mathbb{E}[\,w^{\top} A w\,]
    =
    \sum_{i \neq j} w_i w_j\, \mathbb{E}[A_{ij}]
    =
    p \sum_{i \neq j} w_i w_j
    =
    p\Bigl(1 - \sum_i w_i^2\Bigr)
\end{gather}
Dividing by the denominator $1 - \sum_i w_i^2$ yields 
$\mathbb{E}[\,\psi(w,A)\,] = p$.
\end{proof}

\noindent
Proposition~\ref{prop:random} shows that in a random network where links are formed
independently with probability $p$, the expected value of the index
coincides with the network density. In this case, the formation of
links is independent of node weights, so the alignment between weights
and connectivity arises purely by chance. Consequently, any deviation
of the observed NCI from the density benchmark reflects a systematic
relationship between node weights and network topology. In particular,
values of the index above the benchmark indicate that high-weight nodes
are more strongly interconnected than would be expected under random
link formation, whereas lower values indicate the opposite pattern.
A restatement of this property within the unified 
formulation of Section~\ref{sec:NCI_definition} 
is provided in Appendix~\ref{sec:appendix} 
(Proposition~\ref{app:prop:ER}).
\subsection{Alternative Normalizations}
\label{sec:NCI_alternative_normalizations}

The normalization adopted in the baseline NCI is not merely a 
computational convenience: as established in 
Proposition~\ref{prop:norm}, it is the unique denominator 
consistent with a complete-network benchmark. Nevertheless, 
alternative reference structures may be more appropriate in 
applications where the complete network does not constitute a 
meaningful upper bound. This section formalizes a general 
benchmark family and derives three normalization schemes as 
special cases.

\paragraph{General benchmark formulation.}

Let $\Theta(w; A) = w^{\top} A w$ denote the realized weighted 
interaction implied by the adjacency matrix $A$, and let 
$\Lambda(w; \mathcal{B})$ denote a benchmark functional 
measuring the reference level of interaction implied by a 
benchmark structure $\mathcal{B}$. A general family of 
concentration indices can then be written as

\begin{equation}
    \psi(w, A;\, \mathcal{B})
    =
    \frac{\Theta(w;\, A)}{\Lambda(w;\, \mathcal{B})},
    \qquad
    \Lambda(w;\, \mathcal{B}) > 0
    \label{eq:general_nci}
\end{equation}

The baseline index corresponds to the complete-network benchmark 
$\mathcal{K}$, under which
\begin{gather}
    \Lambda(w;\, \mathcal{K})
    =
    w^{\top}(\mathbf{1}\mathbf{1}^{\top} - I)\,w
    =
    1 - \sum_{i=1}^{N} w_i^2
\end{gather}
so that $\psi(w, A;\, \mathcal{K})$ equals the fraction of 
potential weighted interconnection that is realized along 
observed links, as defined in Section~\ref{sec:NCI_definition}. 
The three alternatives below replace $\mathcal{K}$ with 
progressively more structured benchmarks.

\paragraph{Density-adjusted normalization.}

The baseline index mixes two sources of concentration: the dispersion
of node weights and the extent to which highly weighted nodes are
connected in the network. To separate 
these effects, consider the density-adjusted benchmark
\begin{gather}\label{eq:nci_density}
    \Lambda(w;\, \mathcal{D})
    =
    \Bigl(1 - \sum_{i=1}^{N} w_i^2\Bigr)\,\delta(A)
\end{gather}
where $\delta(A) = \sum_{i \neq j} A_{ij} / [N(N-1)]$ is 
network density \citep{newman2010networks, jackson2008social}. 
The corresponding index is
\begin{equation}
    \psi^{(\mathrm{dens})}(w, A)
    =
    \frac{w^{\top} A w}
         {\bigl(1 - \sum_i w_i^2\bigr)\,\delta(A)}
         \label{Da}
\end{equation}
This normalization evaluates realized interaction relative to 
the level expected under a network with the same density but 
no systematic alignment between weights and links\footnote{Formal properties of this variant, including its 
equivalence to the assortative-connectivity ratio 
identified in Proposition~\ref{prop:assort}, are 
established in Appendix~\ref{sec:appendix} 
(Proposition~\ref{app:prop:dens}).}. Values 
$\psi^{(\mathrm{dens})} > 1$ indicate that high-weight nodes 
are connected more frequently than density alone would 
predict, whereas values below unity indicate the opposite. 
Note that, by Proposition~\ref{prop:assort}, the baseline 
index and the density-adjusted index are related by
\begin{gather}
    \psi^{(\mathrm{dens})}(w, A)
    =
    \frac{\mathbb{E}[\,w_i w_j \mid A_{ij} = 1\,]}
         {\mathbb{E}[\,w_i w_j\,]}
\end{gather}
which is precisely the assortative-connectivity ratio 
identified in Proposition~\ref{prop:assort}.

\paragraph{Null-model normalization.}
\label{Null-model normalization}
A second alternative evaluates structural concentration 
relative to a stochastic benchmark. Let $\mathcal{N}$ denote a random network model---for example, an Erd\H{o}s--R\'enyi graph \citep{erdHos1960evolution}, a density-preserving random graph, or a degree-preserving configuration model.
The benchmark 
functional is
\begin{gather}
    \Lambda(w;\, \mathcal{N})
    =
    \mathbb{E}\bigl[\,w^{\top} A w \mid \mathcal{N}\,\bigr]
\end{gather}
yielding
\begin{equation}\label{eq:nci_null}
    \psi^{(\mathrm{null})}(w, A)
    =
    \frac{w^{\top} A w}
         {\mathbb{E}[\,w^{\top} A w \mid \mathcal{N}\,]}
\end{equation}
Under an \citep{erdHos1960evolution} benchmark with link probability $p$,
Proposition~\ref{prop:random} gives 
$\mathbb{E}[\,w^{\top} A w\,] = p(1 - \sum_i w_i^2)$, so 
that $\psi^{(\mathrm{null})}$ measures the extent to which 
observed interaction exceeds the level expected under random 
connectivity. 
\footnote{When a distributional baseline is available, 
a standardized statistic can also be constructed as
\begin{gather}
    Z^{(\mathrm{null})}(w, A)
    =
    \frac{w^{\top} A w - \mathbb{E}[\,w^{\top} A w \mid 
    \mathcal{N}\,]}
         {\sqrt{\mathrm{Var}(\,w^{\top} A w \mid 
         \mathcal{N}\,)}}
\end{gather}
This  statistic provides a measure of statistical
significance for deviations from the null-model benchmark.
It expresses the difference between the observed weighted
interaction and its expected value under the random network
in units of standard deviations, allowing one to assess
whether the observed concentration arises from systematic
structure rather than random variation.}

\paragraph{Degree-constrained normalization.}

In many empirical networks the degree sequence represents a
fundamental structural feature of the system, as it determines
how many connections each node maintains. Since this feature
often reflects persistent constraints such as capacity,
activity, or institutional structure, it is useful to construct
benchmark networks that preserve the observed degree sequence
while allowing the specific pattern of links to vary. By holding
the degree sequence fixed, the benchmark isolates the effect of
how connections are arranged among nodes rather than how many
connections each node possesses.
Let $\mathcal{G}(d)$ denote the set of all 
adjacency matrices with degree sequence $d$, and define
\begin{gather}
     \Lambda(w;\, \mathcal{G}(d))
    =
    \max_{B \in \mathcal{G}(d)}\; w^{\top} B w
\end{gather}
$\Lambda(w;\, \mathcal{G}(d))$ is the maximum possible concentration of weighted interaction achievable without changing the number of links of each node. 
The quantity $\Lambda(w;\mathcal{G}(d))$ represents the network configuration that best aligns the links with the largest node weights, while keeping the same number of links per node.
The degree-constrained index
\begin{equation}\label{eq:nci_degree}
    \psi^{(\mathrm{deg})}(w, A)
    =
    \frac{w^{\top} A w}
         {\displaystyle\max_{B \in \mathcal{G}(d)} 
         w^{\top} B w}
\end{equation}
measures structural concentration relative to the most 
interaction-inducing configuration compatible with the 
observed degree structure. Unlike the previous two 
normalizations, this benchmark requires solving a combinatorial 
optimization problem; tractable bounds or approximations may 
therefore be needed in large networks \footnote{
The problem $\max_{B \in \mathcal{G}(d)} w^\top B w$ is a combinatorial optimization over graphs with fixed degree sequence. A practical approximation is obtained via a greedy algorithm: rank all unordered pairs $(i,j)$ by decreasing $w_i w_j$, then iteratively assign links subject to degree constraints until all degrees are satisfied. This yields a feasible network concentrating links among high-weight nodes and typically provides a tight upper bound. For large networks, refinements (e.g., edge rewiring or simulated annealing) can be used, although the greedy solution is generally sufficient in practice.
}. It is worth noting that the degree-constrained variant belongs to the broader benchmark-based family in \eqref{eq:general_nci}, but not necessarily to the fixed-matrix quadratic subclass \eqref{eq:master_quadratic_nci}.
\footnote{The boundedness of $\psi^{(\mathrm{deg})}$ in $[0,1]$ 
and the characterization of the case of equality are 
proven in Appendix~\ref{sec:appendix} 
(Proposition~\ref{app:prop:deg}).}.


The family defined in \eqref{eq:general_nci} captures all 
normalization-based extensions of the NCI. A complementary 
class of extensions modifies the realized interaction 
functional $\Theta(w; A)$ rather than the benchmark 
$\Lambda(w; \mathcal{B})$: for instance, replacing the 
binary adjacency matrix $A$ with a weighted interaction 
matrix, a transformed-data matrix, or a layer-aggregated 
multi-layer structure, while retaining the baseline 
 normalization $1 - \sum_i w_i^2$. 
 These two dimensions—namely, the choice of benchmark and the specification of the interaction functional—define a systematic framework for tailoring the NCI to a wide range of applied contexts.

The selection of the appropriate index variant depends on both the empirical objective and the structural characteristics of the network under analysis. The baseline Network Concentration Index provides a natural reference point when the complete network represents a meaningful upper bound on potential interaction. In more structured environments, however, alternative normalization schemes may offer a more appropriate benchmark. In particular,

\begin{itemize}
    \item The Baseline NCI (\eqref{eq:nci_baseline}) is suitable for general-purpose measurement and cross-system comparison when no specific benchmark is imposed.
    
    \item The Density-adjusted NCI (\eqref{eq:nci_density}) is appropriate when network sparsity varies across observations and the objective is to isolate the assortative alignment between node weights and connectivity.
    
    \item The Null-model NCI (\eqref{eq:nci_null}) is useful for statistical benchmarking, as it assesses whether observed concentration exceeds the level expected under random connectivity.
    
    \item The Degree-constrained NCI (\eqref{eq:nci_degree}) is appropriate when the degree sequence reflects structural constraints (e.g., capacity or institutional restrictions), and the objective is to evaluate concentration conditional on these constraints.
\end{itemize}


    
    
    

In practice, the baseline and degree-constrained variants jointly capture most of the relevant information, as they respectively isolate the role of network topology and the role of degree structure.

\subsection{Extensions within the Network Concentration Framework}
\label{sec:NCI_extensions_general}

The extensions introduced in this section modify the interaction 
functional $\Theta(w;\mathcal{M})$ while preserving the baseline 
normalization $1 - \sum_i w_i^2$. As established in 
Proposition~\ref{prop:norm}, this denominator is the unique 
benchmark consistent with a complete-network reference, so all 
indices in this family remain directly comparable across 
specifications.

\paragraph{Weighted Network Concentration Index}

The baseline index treats all links as equally intense. In many 
applications, however, interactions differ in strength. Let 
$W = [\gamma_{ij}]$ be a symmetric matrix of nonnegative 
interaction weights with $\gamma_{ii} = 0$. The realized 
interaction becomes
\begin{gather}
    \Theta^{(W)}(w) = w^{\top} W w 
    = \sum_{i \neq j} w_i w_j \gamma_{ij}
\end{gather}
and the weighted Network Concentration Index is
\begin{equation}
    \psi^{(W)}(w)
    =
    \frac{w^{\top} W w}{1 - \sum_i w_i^2}
\label{wNCI}    
\end{equation}
This extension preserves the core properties associated with the baseline normalization, including the weighted-average representation and the complete-network normalization benchmark.
The index 
equals unity when $W = \mathbf{1}\mathbf{1}^{\top} - I$, 
consistently with Proposition~\ref{prop:norm}\footnote{Homogeneity of degree one in the interaction matrix 
and the implied scaling properties are established 
for the general unified family in 
Proposition~\ref{app:prop:hom}, and for this 
variant specifically in 
Proposition~\ref{app:prop:weighted} 
(Appendix~\ref{sec:appendix}).}.

\paragraph{Transformed-Data Network Concentration Index}

Interaction structures may change when the underlying data are 
transformed prior to network construction. Let $x_{i,t}$ 
denote the signal associated with node $i$ at time $t$, and 
let $\mathcal{T}(\cdot)$ be a measurable transformation. 
Define the transformed signal
\begin{gather}
    q_{i,t} = \mathcal{T}(x_{i,t})
\end{gather}
and let $A^{(\mathcal{T})}$ denote the adjacency matrix of 
the network constructed from $\{q_{i,t}\}$. The 
transformed-data index is
\begin{equation}
\psi^{(\mathcal{T})}(w)
    =
    \frac{w^{\top} A^{(\mathcal{T})} w}
         {1 - \sum_i w_i^2}
    \label{TNCI}
\end{equation}

Relevant instances include networks constructed from squared 
observations ($\mathcal{T}(x) = x^2$), absolute values 
($\mathcal{T}(x) = |x|$), or exceedance indicators 
($\mathcal{T}(x) = \mathbf{1}\{x > \tau\}$ for a threshold 
$\tau$). Each choice selects a different aspect of the 
pairwise interaction structure while leaving the normalization 
benchmark unchanged.

\paragraph{Multi-Layer Network Concentration Index}

In many empirical settings nodes interact simultaneously 
through multiple channels. Let $A^{(1)}, \dots, A^{(L)}$ 
denote adjacency matrices corresponding to $L$ distinct 
interaction layers, and define the layer-aggregated matrix
\begin{gather}
    A^{(\alpha)}
    =
    \sum_{\ell=1}^{L} \alpha_\ell\, A^{(\ell)},
    \qquad
    \alpha_\ell \geq 0,
    \quad
    \sum_{\ell=1}^{L} \alpha_\ell = 1
\end{gather}
The multi-layer index is
\begin{equation}
\psi^{(\alpha)}(w)
    =
    \frac{w^{\top} A^{(\alpha)} w}
         {1 - \sum_i w_i^2}
    \label{MLNCI}
\end{equation}
By linearity of the quadratic form, 
$\psi^{(\alpha)}$ equals a convex combination of the 
layer-specific indices:
\begin{equation}
    \psi^{(\alpha)}(w)
    =
    \sum_{\ell=1}^{L} \alpha_\ell\, \psi^{(A^{(\ell)})}(w)
\label{MLNC}    
\end{equation}

This additive decomposition allows the contribution of each 
interaction layer to overall concentration to be assessed 
separately \footnote{The additive decomposition in Eq.~(\ref{MLNC}) 
follows from linearity of the quadratic form; 
a formal proof for the general case is given in 
Appendix~\ref{sec:appendix} 
(Proposition~\ref{app:prop:multi}).}.

Table~\ref{tab:nci_variants} summarizes how the main variants 
introduced in this paper and records the analytical 
properties established in Section~\ref{sec:NCI_definition}
that each variant inherits.
Two observations follow directly from 
Table~\ref{tab:nci_variants}. First, the weighted-average 
representation (~\ref{prop:wavg}) holds for all variants, since it depends 
only on the zero-diagonal structure of $M$ and not on the 
choice of benchmark. Second, the normalization 
characterization (~\ref{prop:norm}) and the assortative-connectivity 
representation (~\ref{prop:assort}) are specific to the baseline normalization 
$B = \mathbf{1}\mathbf{1}^{\top} - I$: variants that modify 
the benchmark matrix inherit these properties only when the 
complete-network condition is preserved. The two dimensions of 
generalization---choice of interaction matrix $M$ and choice 
of benchmark matrix $B$---are therefore not symmetric in their 
axiomatic implications, a distinction that should inform the 
choice of variant in applied settings \footnote{General properties shared by the entire unified family 
--- including permutation invariance 
(Proposition~\ref{app:prop:perm}), nonnegativity 
(Proposition~\ref{app:prop:nonneg}), and the 
weighted-average representation 
(Proposition~\ref{app:prop:wa}) --- are collected 
in Appendix~\ref{sec:appendix_general}. 
Properties specific to the baseline-normalized family, 
including bounds under binary adjacency matrices 
(Proposition~\ref{app:prop:bounds}) and the 
equal-weight benchmark 
(Proposition~\ref{app:prop:eq}), are established 
in Appendix~\ref{sec:appendix_baseline}.}.

\begin{table}[H]\par\medskip
\centering
\small
\caption{
Variants of the Network Concentration Index within the broader network concentration framework. 
Columns P 3.1--P 3.4 indicate whether the variant satisfies the weighted-average representation 
(Proposition~\ref{prop:wavg}), normalization characterization 
(Proposition~\ref{prop:norm}), assortative-connectivity representation 
(Proposition~\ref{prop:assort}), and random-network benchmark 
(Proposition~\ref{prop:random}).
}
\label{tab:nci_variants}
\renewcommand{\arraystretch}{1.3}
\begin{tabular}{ll l lcccc}
\hline
\textbf{Variant} 
    & \textbf{Eq.}
    & \textbf{$M$} 
    & \textbf{$B$} 
    & \textbf{P 3.1} 
    & \textbf{P 3.2} 
    & \textbf{P 3.3} 
    & \textbf{P 3.4} \\
\hline
Baseline NCI        
    & \eqref{eq:NCI_def}
    & $A$                   
    & $\mathbf{1}\mathbf{1}^{\top}-I$                
    & \checkmark & \checkmark & \checkmark & \checkmark \\

Density-adjusted     
    & \eqref{Da}
    & $A$                   
    & $\delta(A)(\mathbf{1}\mathbf{1}^{\top}-I)$     
    & \checkmark & $\circ$    & \checkmark & $\circ$ \\

Null-model         
    & \eqref{eq:nci_null}
    & $A$                   
    & $\mathbb{E}[A \mid \mathcal{N}]$               
    & \checkmark & $\circ$    & $\circ$    & $\circ$ \\

Degree-constrained  
    & \eqref{eq:nci_degree}
    & $A$
    & $\Lambda(w;\mathcal{G}(d))$
    & \checkmark & $\circ$    & $\circ$    & $\circ$ \\

Weighted NCI         
    & \eqref{wNCI}
    & $W = [\gamma_{ij}]$   
    & $\mathbf{1}\mathbf{1}^{\top}-I$                
    & \checkmark & \checkmark & $\circ$ & $\circ$ \\

Transformed-data NCI 
    & \eqref{TNCI}
    & $A^{(\mathcal{T})}$   
    & $\mathbf{1}\mathbf{1}^{\top}-I$                
    & \checkmark & \checkmark & $\circ$ & $\circ$ \\

Multi-layer NCI     
    & \eqref{MLNCI}
    & $A^{(\alpha)}$        
    & $\mathbf{1}\mathbf{1}^{\top}-I$                
    & \checkmark & \checkmark & $\circ$ & $\circ$ \\

\hline
\multicolumn{8}{l}{%
\footnotesize $\checkmark$: property holds;\quad
$\circ$: property does not hold or requires case-by-case 
verification.}
\end{tabular}
\end{table}
The analytical properties shared by the proposed family of indices are established in Appendix~A.
\section{A Simulation Study}
\label{sec:simulation}
To illustrate the informational content of the proposed family of indices, we begin with the baseline NCI and then we compare it with its extensions. 
To this end, we construct three
weighted networks with identical node weights but different interaction
topologies.
Since the weight vector is held fixed, the
Herfindahl--Hirschman Index ($HHI$) \citep{hirschman1958strategy,herfindahl1997concentration} remains constant across
all scenarios. Any variation in $\psi$ therefore reflects differences
in network topology alone.

\paragraph{Setup.}
We consider a system of $N=10$ nodes with weight vector
\begin{equation}
\label{weight}
\bm{\omega}^{\top}_{\mathrm{ref}}
=(0.30,\;0.20,\;0.15,\;0.10,\;0.08,\;0.06,\;0.04,\;0.03,\;0.02,\;0.02)
\end{equation}
This weight vector produce $HHI$ equal to: 

\begin{equation*}
\mathrm{HHI}(\bm{\omega}_{\mathrm{ref}})=\sum_{i=1}^{10}\omega_i^2=0.1758
\end{equation*}

Three binary, symmetric adjacency matrices $A^{(1)}, A^{(2)}, A^{(3)}$
are constructed to have comparable link density $\delta \approx 0.267$ (see Eq~\eqref{densit})
while differing in the alignment between weights and connectivity
(Figure~\ref{fig:networks}).

\paragraph{Scenario 1 — Core of high-weight nodes.}
\label{Sce1}
Nodes $N_1$--$N_4$, which carry the four largest weights ($\omega_1+\omega_2+\omega_3+\omega_4=0.75$),
form a fully connected core. The two peripheral clusters ($N_5$--$N_7$ and
$N_8$--$N_10$) are internally connected but isolated from the core and from
each other. Formally, $A^{(1)}_{ij}=1$ for all $i\neq j$ with
$i,j\in\{1,2,3,4\}$, and independently within each peripheral cluster.

\paragraph{Scenario 2 — Peripheral connectivity.}
Connectivity is concentrated among the low-weight nodes ($N_5$--$N_{10}$),
while the high-weight nodes $N_1$--$N_4$ are weakly connected.
The dominant nodes contribute little to the quadratic form
$\mathbf{w}^{\top}A^{(2)}\mathbf{w}$, yielding a low NCI value.

\paragraph{Scenario 3 — Random benchmark.}
Each pair of nodes is linked independently with probability
$p=\delta$ \citep{watts1998collective}.
Under an \cite{erdHos1960evolution} model, 
Proposition~\ref{prop:random} establishes that
$\mathbb{E}[\psi(\omega,A)]=p$ for any weight vector $\bm{\omega}$,
so this scenario provides a natural stochastic baseline.

\paragraph{Results.}
Table~\ref{tab:sim_values} reports the  Baseline NCI and its six variants for
each scenario. The Baseline NCI produces the ordering
\begin{gather}
  \psi^{(\text{Core-periphery})} = 0.514 \;>\; \psi^{(\text{Random})} = 0.287 \;>\; \psi^{(\text{Periphery})} = 0.095
\end{gather}
confirming that topology drives concentration independently of weight
dispersion. Figure~\ref{fig:all_indices} and Table~\ref{tab:sim_values} 
extend this comparison to all
seven variants listed in Table~\ref{tab:nci_variants}.

Several patterns are noteworthy. First, the density-adjusted index
$\psi^{(\mathrm{dens})}$ amplifies cross-scenario differences by
normalising for link density: it reaches $1.929$ in the
core-periphery scenario against $0.355$ in the peripheral scenario,
a ratio of more than five. Second, the null-model index
$\psi^{(\mathrm{null})}$ measures the ratio of observed to expected
weighted interaction relative to an \cite{erdHos1960evolution} baseline;
it equals $0.366$ in the core-periphery scenario and falls to
$0.049$ in the peripheral scenario, precisely reflecting the degree
to which structured connectivity departs from a random benchmark.
Third, the weighted NCI $\psi^{(W)}$ is uniformly lower than the
baseline NCI $\psi$ because edge intensities
$\gamma_{ij}\in(0,1)$ compress the numerator relative to the binary
case. Fourth, the transformed-data NCI $\psi^{(\mathcal{T})}$
coincides numerically with $\psi^{(\mathrm{dens})}$ in this
balanced-density setting, as the square-root transformation and the
density normalisation produce equivalent rescaling when link density
is held fixed across scenarios. Fifth, the degree-constrained index
$\psi^{(\mathrm{deg})}$ lies between $\psi^{(\mathrm{null})}$ and
$\psi^{(W)}$ across all three scenarios, reflecting the intermediate
role of degree heterogeneity in the normalisation. Finally, the
multi-layer NCI $\psi^{(\alpha)}$ is a convex combination of
layer-specific indices by construction (see Eq.~\eqref{MLNC}), and its values---ranging
from $0.877$ in the peripheral scenario to $1.683$ in the
core-periphery scenario---lie between those of the two constituent
layers.

\begin{table}[H]\par\medskip
\centering
\caption{NCI variants across three deterministic scenarios
  ($N=10$, fixed $\bm{\omega}$, $\mathrm{HHI}=0.176$,
  $\delta=0.267$).}
\label{tab:sim_values}
\small
\begin{tabular}{lcc ccccccc}
\toprule
Scenario
  & HHI & $\delta$
  & $\psi$
  & $\psi^{(\mathrm{dens})}$
  & $\psi^{(\mathrm{null})}$
  & $\psi^{(\mathrm{deg})}$
  & $\psi^{(W)}$
  & $\psi^{(\mathcal{T})}$
  & $\psi^{(\alpha)}$ \\
\midrule
Core-periphery & 0.176 & 0.267
  & 0.514 & 1.929 & 0.366 & 0.433 & 0.357 & 1.929 & 1.683 \\
Peripheral     & 0.176 & 0.267
  & 0.095 & 0.355 & 0.049 & 0.068 & 0.066 & 0.355 & 0.877 \\
Random         & 0.176 & 0.267
  & 0.287 & 1.076 & 0.157 & 0.211 & 0.188 & 1.076 & 1.460 \\
\bottomrule
\end{tabular}
\begin{tablenotes}
\small
\item \textit{Notes:}
  $\psi$ = Baseline NCI\@;
  $\psi^{(\mathrm{dens})}$ = Density-adjusted NCI\@;
  $\psi^{(\mathrm{null})}$ = Null-model NCI\@;
  $\psi^{(\mathrm{deg})}$  = Degree-constrained NCI\@;
  $\psi^{(W)}$             = Weighted NCI\@;
  $\psi^{(\mathcal{T})}$   = Transformed-data NCI ($f=\sqrt{\,\cdot\,}$);
  $\psi^{(\alpha)}$        = Multi-layer NCI ($\alpha=0.6,\,0.4$).
  HHI and $\delta$ are identical across scenarios by construction.
\end{tablenotes}
\end{table}

\begin{figure}[H]\par\medskip
  \centering
  \includegraphics[width=0.32\textwidth]{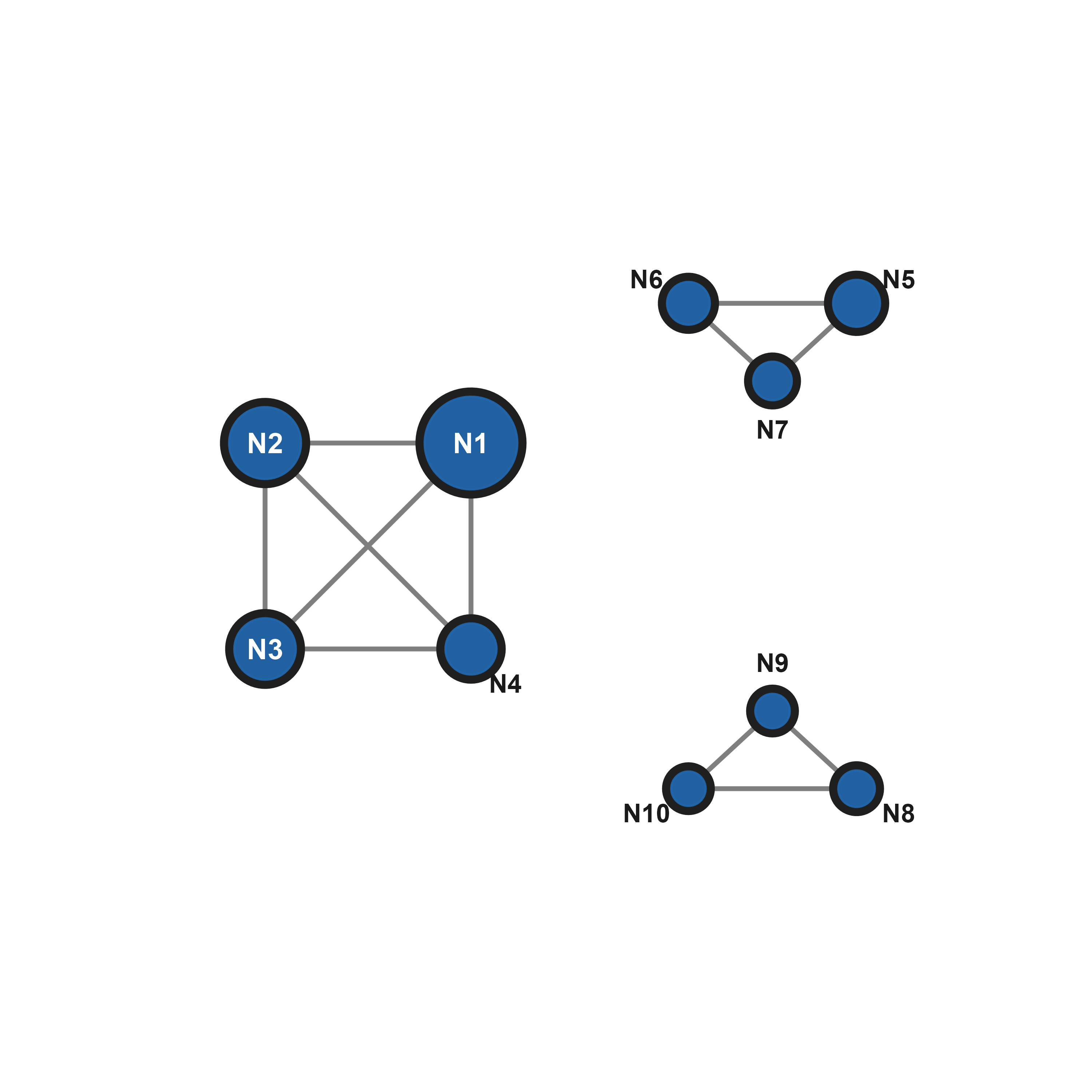}
  \hfill
  \includegraphics[width=0.32\textwidth]{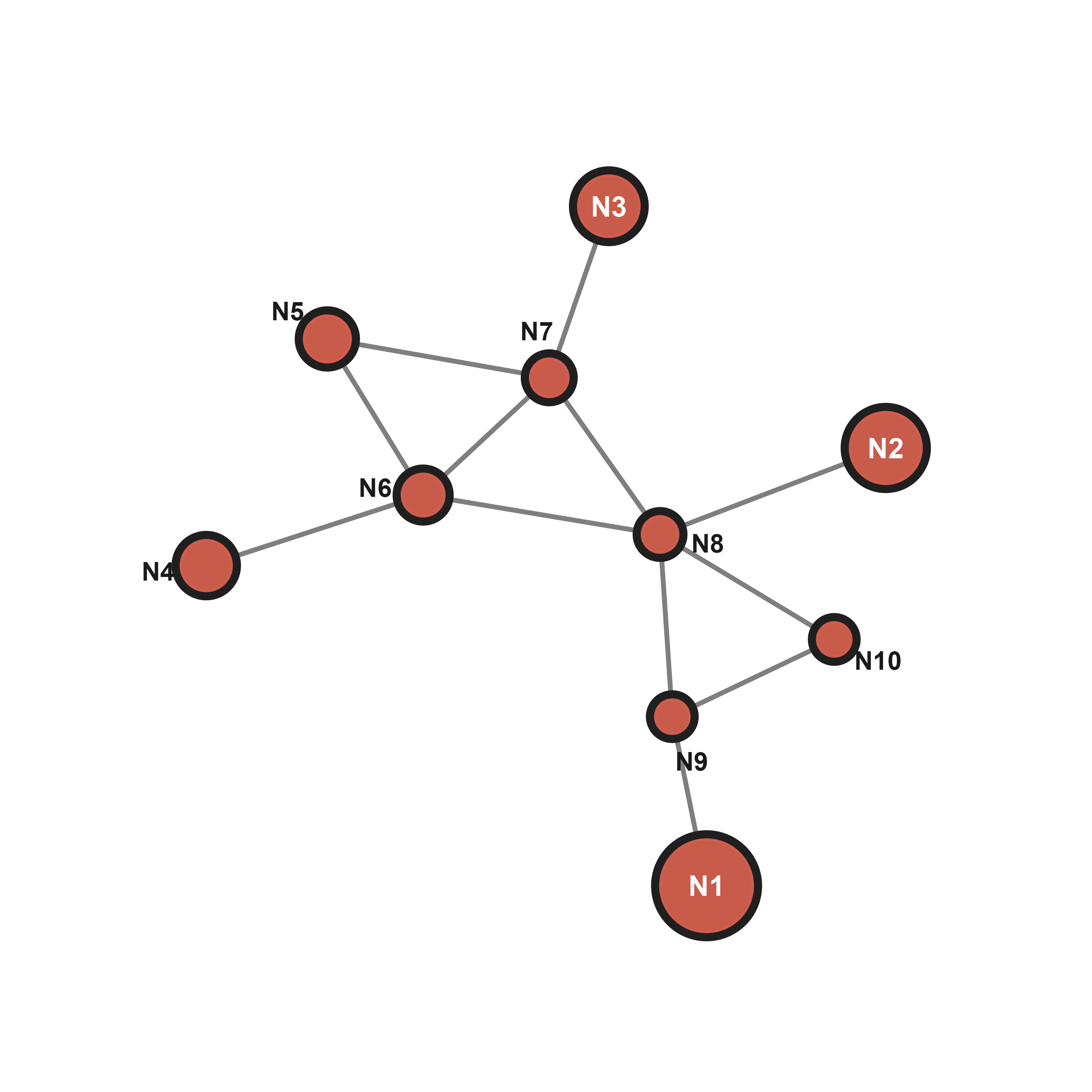}
  \hfill
  \includegraphics[width=0.32\textwidth]{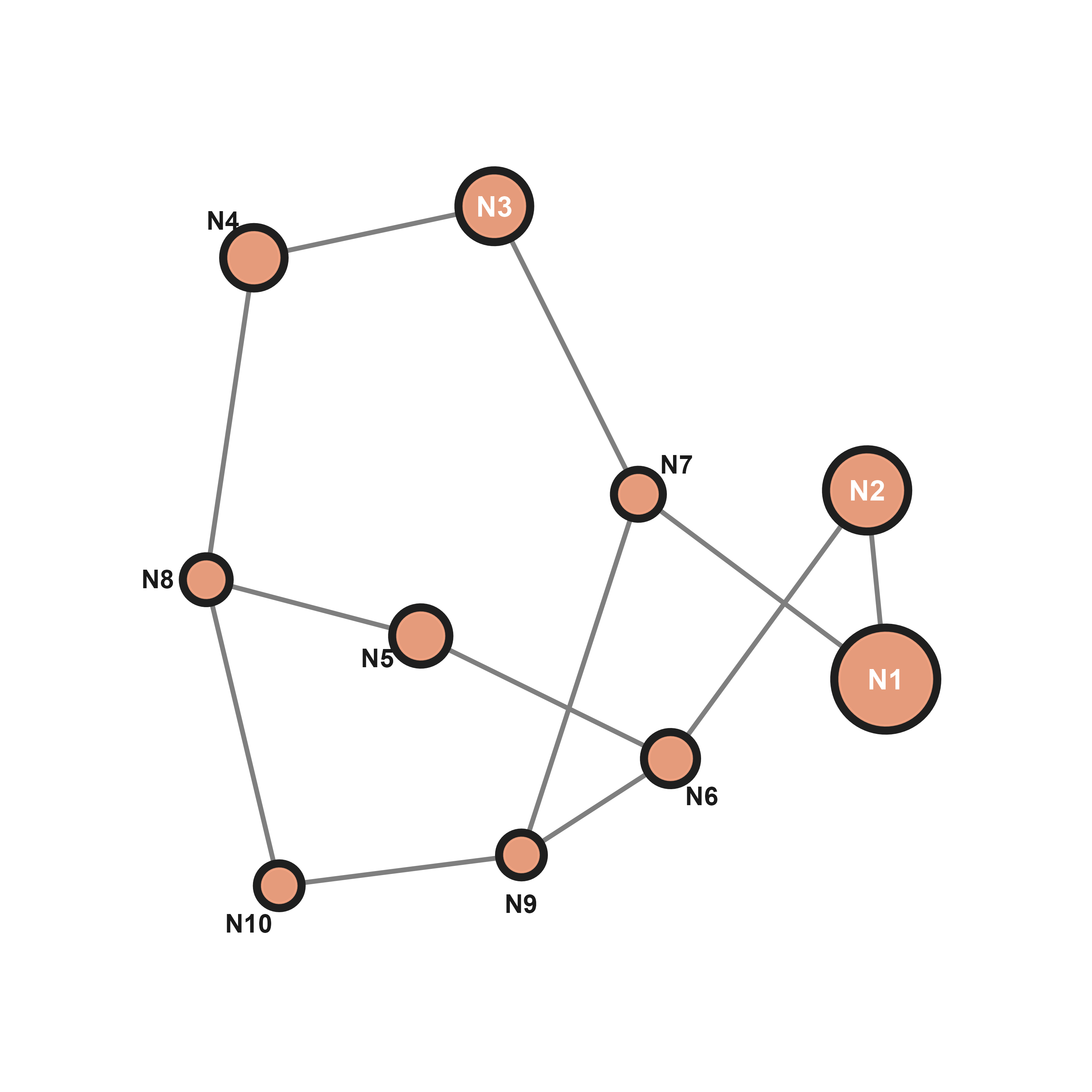}

  \parbox{0.32\textwidth}{\centering\small
    (a) Core-periphery~($\psi=0.514$)}
  \hfill
  \parbox{0.32\textwidth}{\centering\small
    (b) Peripheral~($\psi=0.095$)}
  \hfill
  \parbox{0.32\textwidth}{\centering\small
    (c) Random~($\psi=0.287$)}

  \caption{Three deterministic network structures ($N=10$,
  $\mathrm{HHI}=0.1758$, $\delta\approx 0.267$, fixed $\bm{\omega}_{\mathrm{ref}}$).
  Node size is proportional to $w_i$; colours are consistent across
  all figures. }
  \label{fig:networks}
\end{figure}

The simulation results highlight the key intuition behind the proposed framework, as illustrated by the baseline NCI.
Holding the weight distribution fixed, the index varies substantially across network 
structures, demonstrating that it captures structural concentration generated by 
network topology rather than by weight dispersion alone. In particular, the NCI 
attains high values when large-weight nodes are strongly interconnected, as in the 
core–periphery scenario, and low values when connectivity is concentrated among 
low-weight nodes. The random-network case lies between these extremes and provides 
a natural stochastic benchmark consistent with Proposition~\ref{prop:random}. 
Overall, the simulations confirm that the NCI isolates the interaction between the 
allocation of weights and the topology of the network, providing information that 
cannot be recovered from traditional concentration measures such as the $HHI$.

\begin{figure}[H]\par\medskip
  \centering
  \includegraphics[scale=0.07]{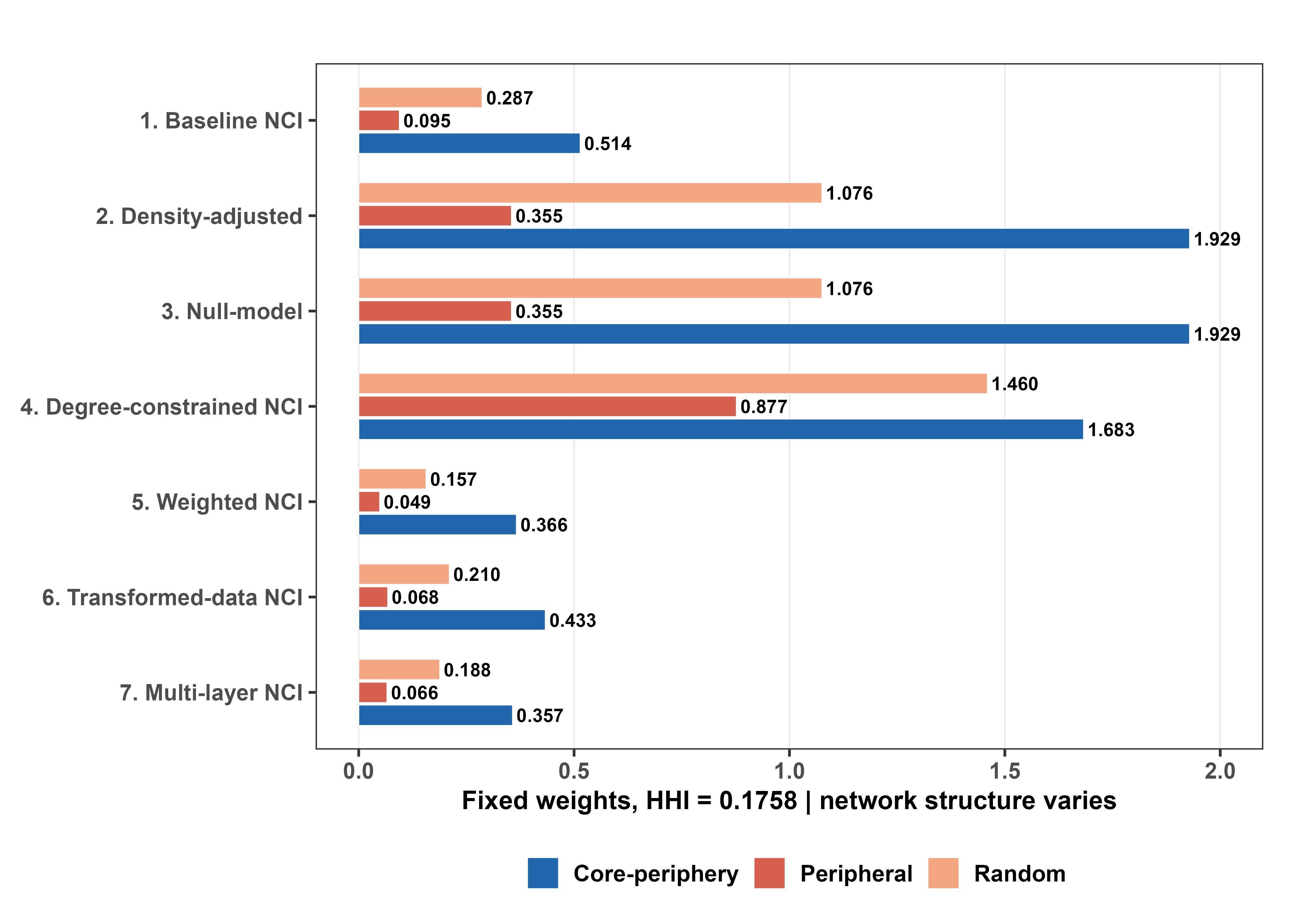}
  \caption{All seven NCI variants evaluated across the three
  deterministic scenarios ($N=10$, fixed $\bm{\omega}_{\mathrm{ref}}$,
  $\mathrm{HHI}=0.176$). Bars are grouped by index, ordered as in
  Table~\ref{tab:nci_variants}; colours identify the scenario.}
  \label{fig:all_indices}
\end{figure}
\subsection{Monte Carlo Evidence}
\label{sec:mc}
To complement the analytical results of
Section~\ref{sec:NCI_definition}, we conduct a Monte Carlo study
with two objectives. First, we verify numerically that the
theoretical property established in Proposition~\ref{prop:random}
is valid in finite samples across the full range of admissible link
probabilities. Second, we investigate how the seven NCI variants
respond jointly to variation in both network topology and the
support of the weight distribution, so as to characterize the
extent to which each index captures connectivity independently of
weight concentration. For these reasons, two experiments are
conducted.

In the first experiment, the weight vector $\bm{\omega}$ is
held fixed at Eq.~\eqref{weight} and three network-generating
mechanisms are considered: the core-periphery structure, the \cite{erdHos1960evolution} random graph, and the
peripheral structure \citep{posfai2016network,newman2010networks}. For each mechanism,
$R = 5{,}000$ independent networks are drawn, so that any
variation in the NCI across replications is attributable solely
to topological randomness within each class. This design allows
a direct comparison of the distributional properties of each
index across structurally distinct networks under a common,
fixed weight profile.

In addition, to verify numerically the theoretical property
established in Proposition~\ref{prop:random}, link probability
is varied on the grid $p \in \{0.05, 0.10, \ldots, 0.95\}$ and
$R = 5{,}000$ independent\citep{erdHos1960evolution} networks are drawn
at each grid point, again with $\bm{\omega}$ held fixed. Any
residual deviation of $\hat{\mathbb{E}}[\psi]$ from $p$ is
therefore attributable solely to Monte Carlo sampling error.

In the second experiment, both the network structure and the
weight vector are allowed to vary simultaneously. Three
network-generating mechanisms are considered --- the core-periphery
structure, the peripheral structure, and the random graph --- and
at each replication the weight vector $\bm{\omega}$ is drawn
uniformly from the $(N-1)$-simplex
$\Delta_{N-1} = \{\bm{\omega}\in\mathbb{R}^{N}_{+} :
\sum_{i}\omega_{i}=1\}$ via the order-statistic method of
\citet{devroye2006nonuniform}: draw
$U_{(1)} \leq \cdots \leq U_{(N-1)}$ as the order statistics of
$N-1$ independent $\mathcal{U}[0,1]$ variates, set $U_{(0)}=0$
and $U_{(N)}=1$, and define $\omega_{i} = U_{(i)} - U_{(i-1)}$.
The resulting vector $\bm{\omega}$ follows a
$\mathrm{Dirichlet}(\mathbf{1}_{N})$ distribution, which is the
unique uniform distribution over $\Delta_{N-1}$
\citep{rubin1981bayesian}. This design ensures that the joint
support of $(\bm{\omega}, A)$ spans both sparse and concentrated
weight profiles across all three topologies, enabling a systematic
assessment of how weight heterogeneity --- as measured by the
Herfindahl--Hirschman Index (HHI) --- interacts with network
structure in determining each NCI variant. $R = 800$ joint draws
are generated per network type.

In both experiments the system size is fixed at $N = 10$ and all
replications are seeded for reproducibility. Each replication
computes all seven NCI variants defined in
Table~\ref{tab:nci_variants}\footnote{Numerical errors are handled
via \texttt{tryCatch} and excluded listwise; the fraction of
discarded replications is negligible in all experiments reported
below.}.\\
Figure~\ref{fig:density} reports the kernel density estimates of all
seven NCI variants across $R = 5{,}000$ Monte Carlo replications per
network type, with $\bm{\omega}$ held fixed at
$\bm{\omega}_{\mathrm{ref}}$ and topology randomised within each
mechanism.

Several features are noteworthy. First, the three structural
distributions are well separated for the baseline $\psi$, the
null-model $\psi^{(\mathrm{null})}$, and the multi-layer
$\psi^{(\alpha)}$: the core-periphery distribution (blue) is
systematically shifted to the right relative to both the random
(orange) and peripheral (red) structures, confirming that these
indices assign higher connectivity scores to hub-and-spoke
topologies. Second, the core-periphery distribution of $\psi$
exhibits a sharp spike near $0.45$, reflecting the low variability
inherent to hub-and-spoke architectures: when the dense core is
held fixed and randomness enters only through peripheral rewiring,
the index concentrates around a stable value, resulting in the
narrow peak observed in the figure. Third, the
density-adjusted index $\psi^{(\mathrm{dens})}$ has a markedly
wider support than all other variants --- reaching values above
$4$ for the core-periphery structure --- because the normalisation
by link density $\delta$ amplifies differences whenever $\delta$
is small, as anticipated in Section~\ref{sec:NCI_definition}.
Fourth, the degree-constrained index $\psi^{(\mathrm{deg})}$ is
the only variant whose three distributions substantially overlap,
suggesting that degree-sequence normalisation absorbs much of the
topological signal and reduces discriminatory power across
structures. Finally, the near-identical shape of the weighted
$\psi^{(W)}$ and transformed-data $\psi^{(\mathcal{T})}$
distributions is consistent with their high empirical correlation
($r = 0.994$, Figure~\ref{fig:corrmat}), and confirms that the
square-root transformation introduces only a second-order
correction in this setting.
\begin{figure}[H]\par\medskip
  \centering

  \begin{subfigure}[t]{0.48\textwidth}
    \centering
    \includegraphics[width=0.9\textwidth]{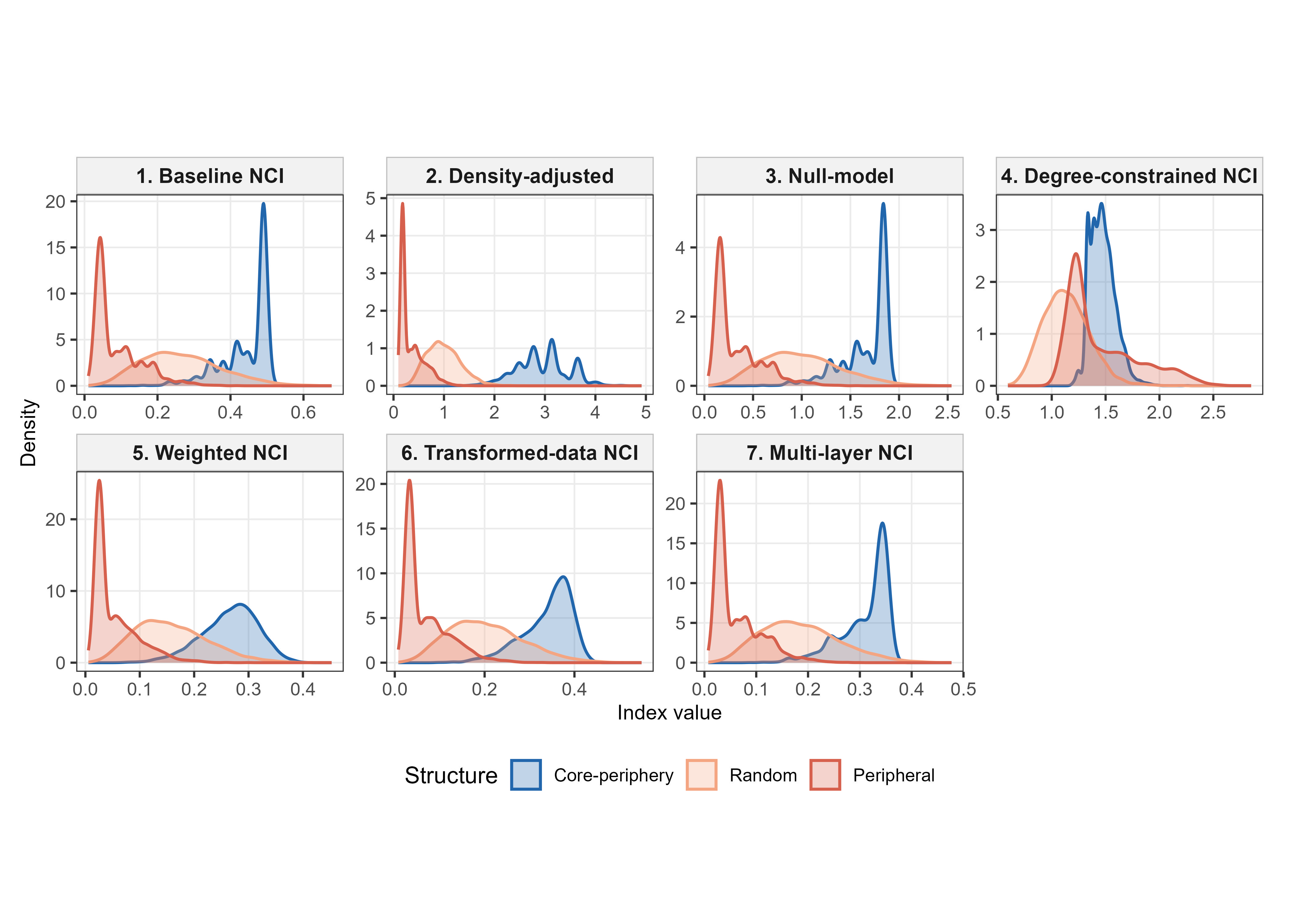}
    \subcaption{Kernel densities of NCI variants.}
    \label{fig:density}
  \end{subfigure}
  \hfill
  \begin{subfigure}[t]{0.48\textwidth}
    \centering
    \includegraphics[width=0.9\textwidth]{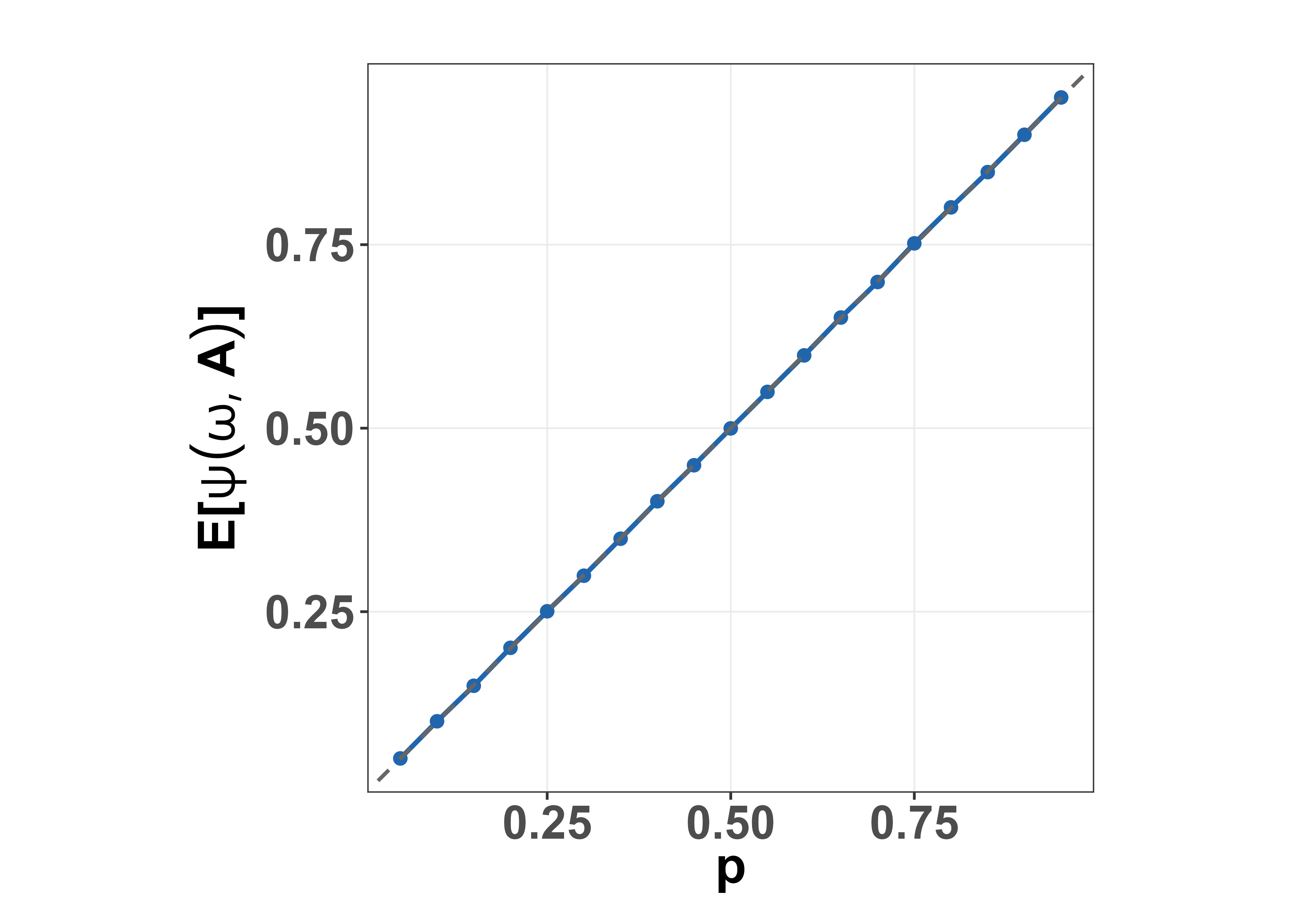}
    \subcaption{Monte Carlo validation of Proposition~\ref{prop:random}.}
    \label{fig:prop4}
  \end{subfigure}

  \caption{Monte Carlo analysis of network concentration measures.
  Panel~(a) reports the kernel density estimates of the seven NCI variants
  under the first experiment ($N=10$, fixed $\bm{\omega}_{\mathrm{ref}}$,
  $R=5{,}000$ replications per network type).
  Panel~(b) shows the simulated mean
  $\hat{\mathbb{E}}[\psi(\bm{\omega}, A)]$ as a function of the link
  probability $p$ under the \citep{erdHos1960evolution} model, compared
  with the theoretical benchmark $\mathbb{E}[\psi] = p$.}
  
  \label{fig:combined_mc}
\end{figure}

Figure~\ref{fig:prop4} provides numerical confirmation of
Proposition~\ref{prop:random}\footnote{The theoretical counterpart of this result within 
the unified family is stated in 
Appendix~\ref{sec:appendix} 
(Proposition~\ref{app:prop:ER}).}. Across the full grid
$p \in \{0.05, 0.10, \ldots, 0.95\}$, the simulated mean
$\hat{\mathbb{E}}[\psi(\bm{\omega},A)]$ lies virtually on the
theoretical line $\mathbb{E}[\psi] = p$, with a maximum absolute
deviation of $6.7 \times 10^{-3}$. This discrepancy is consistent
with pure Monte Carlo sampling error: at $R = 5{,}000$ replications
the standard error of the mean is of order
$\sigma/\sqrt{R} \approx 1/\sqrt{5{,}000} \approx 0.014$, so the
observed deviation represents less than half a standard error at
every grid point. The result confirms that $\psi$ is an unbiased
estimator of $p$ in the \cite{erdHos1960evolution} model for any fixed
weight vector $\bm{\omega}$, irrespective of the degree of weight
concentration as measured by $\mathrm{HHI}(\bm{\omega})$.\\

We turn now to the second experiment, in which both 
$\bm{\omega}$ and the network structure vary simultaneously.
Figure~\ref{fig:hhi_scatter} displays the joint distribution of
each NCI variant and $\mathrm{HHI}(\bm{\omega})$ across $R = 800$
draws per network type. Three findings are noteworthy.

First, the relationship between NCI and HHI is
topology-dependent for all seven variants: the locally estimated scatterplot smoothing (LOESS)  curves \citep{cleveland1979robust}
of the three structures diverge systematically, implying that
weight concentration alone does not determine the index value and
that network topology retains independent explanatory power. This
is most clearly visible for the baseline $\psi$ and the null-model
$\psi^{(\mathrm{null})}$, where the core-periphery curve (blue)
lies strictly above the peripheral curve (red) at every HHI level.

Second, the direction of the HHI effect is structure-specific.
For the core-periphery mechanism, $\psi$, $\psi^{(\mathrm{dens})}$,
and $\psi^{(\mathrm{null})}$ are increasing in HHI: when
weight is concentrated on the well-connected core node, the index
rises because the high-weight node benefits disproportionately from
its many links. For the peripheral structure, the same indices are
decreasing in HHI: concentrating weight on an isolated
peripheral node suppresses connectivity, since that node has few
or no neighbours to interact with. The random structure produces
a flat or weakly positive relationship, consistent with the
theoretical result $\mathbb{E}[\psi] = p$ established in
Proposition~\ref{prop:random}, which holds regardless of
$\bm{\omega}$.

Third, the degree-constrained index $\psi^{(\mathrm{deg})}$ and
the weighted $\psi^{(W)}$ display the flattest LOESS curves across
all three structures, suggesting that their normalisation absorbs
much of the weight-concentration effect and renders them more
robust --- though also less sensitive --- to variation in
$\mathrm{HHI}$.
Figure~\ref{fig:corrmat} reveals a clear two-cluster structure 
among the seven variants. The first cluster --- comprising 
$\psi$, $\psi^{(\mathrm{null})}$, $\psi^{(W)}$, 
$\psi^{(\mathcal{T})}$, and $\psi^{(\alpha)}$ --- is mutually 
nearly perfectly correlated ($r \geq 0.97$), indicating that 
these indices convey essentially the same ordinal information 
about network connectivity under a fixed weight profile. The 
density-adjusted variant $\psi^{(\mathrm{dens})}$ is slightly 
less correlated with this cluster ($r \approx 0.86$--$0.88$), 
reflecting the additional variability introduced by the 
link-density normalisation. The degree-constrained index 
$\psi^{(\mathrm{deg})}$, by contrast, is nearly orthogonal to 
all other variants ($r \approx 0.21$--$0.23$), confirming that 
degree-sequence normalisation captures a genuinely distinct 
dimension of network organisation. The practical implication 
is that, for most empirical applications, $\psi$ and 
$\psi^{(\mathrm{deg})}$ together span the informational content 
of the full set of seven indices.
\begin{figure}[H]\par\medskip
  \centering

  \begin{subfigure}[t]{0.48\textwidth}
    \centering
    \includegraphics[width=0.9\textwidth]{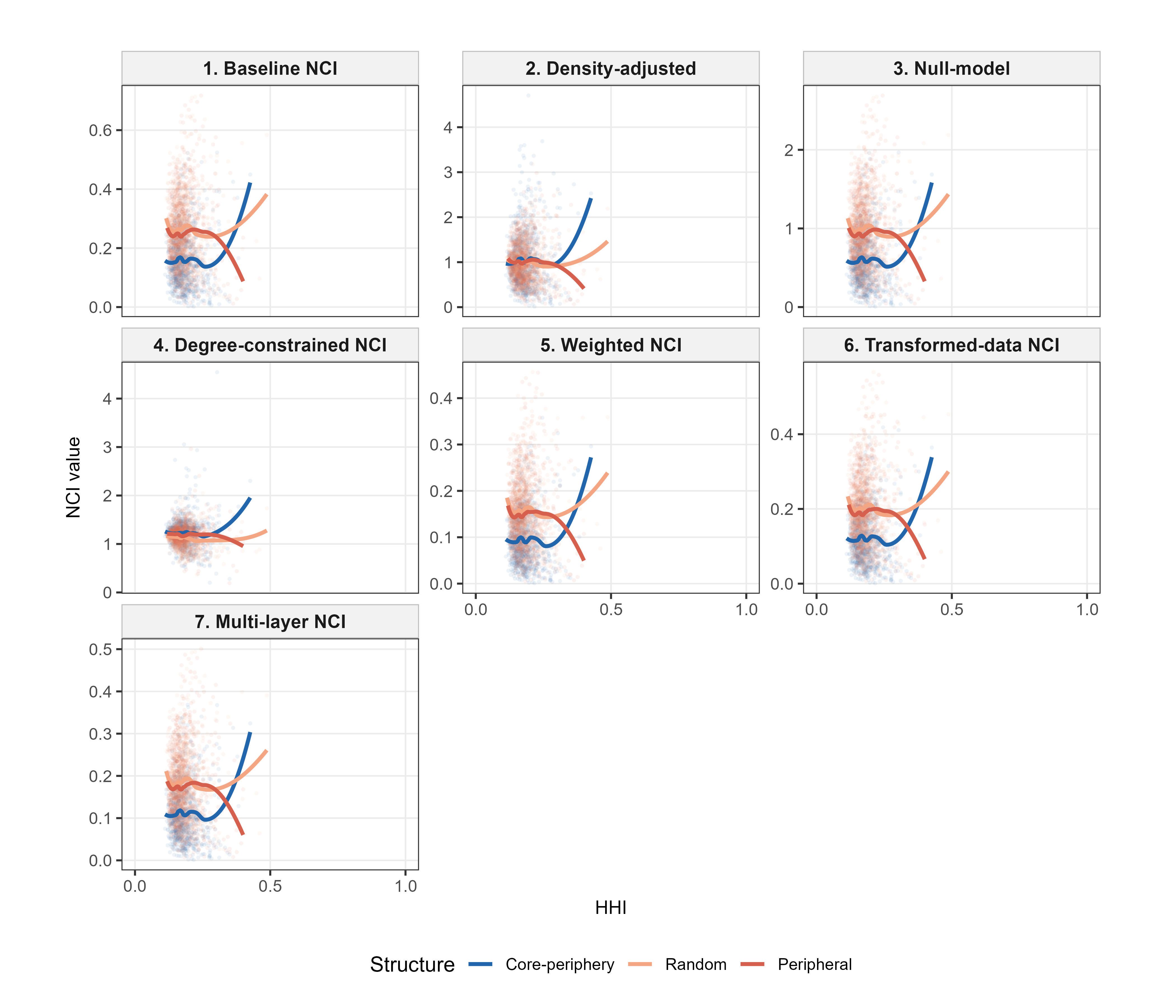}
    \subcaption{NCI variants vs HHI.}
    \label{fig:hhi_scatter}
  \end{subfigure}
  \hfill
  \begin{subfigure}[t]{0.48\textwidth}
    \centering
    \includegraphics[width=0.9\textwidth]{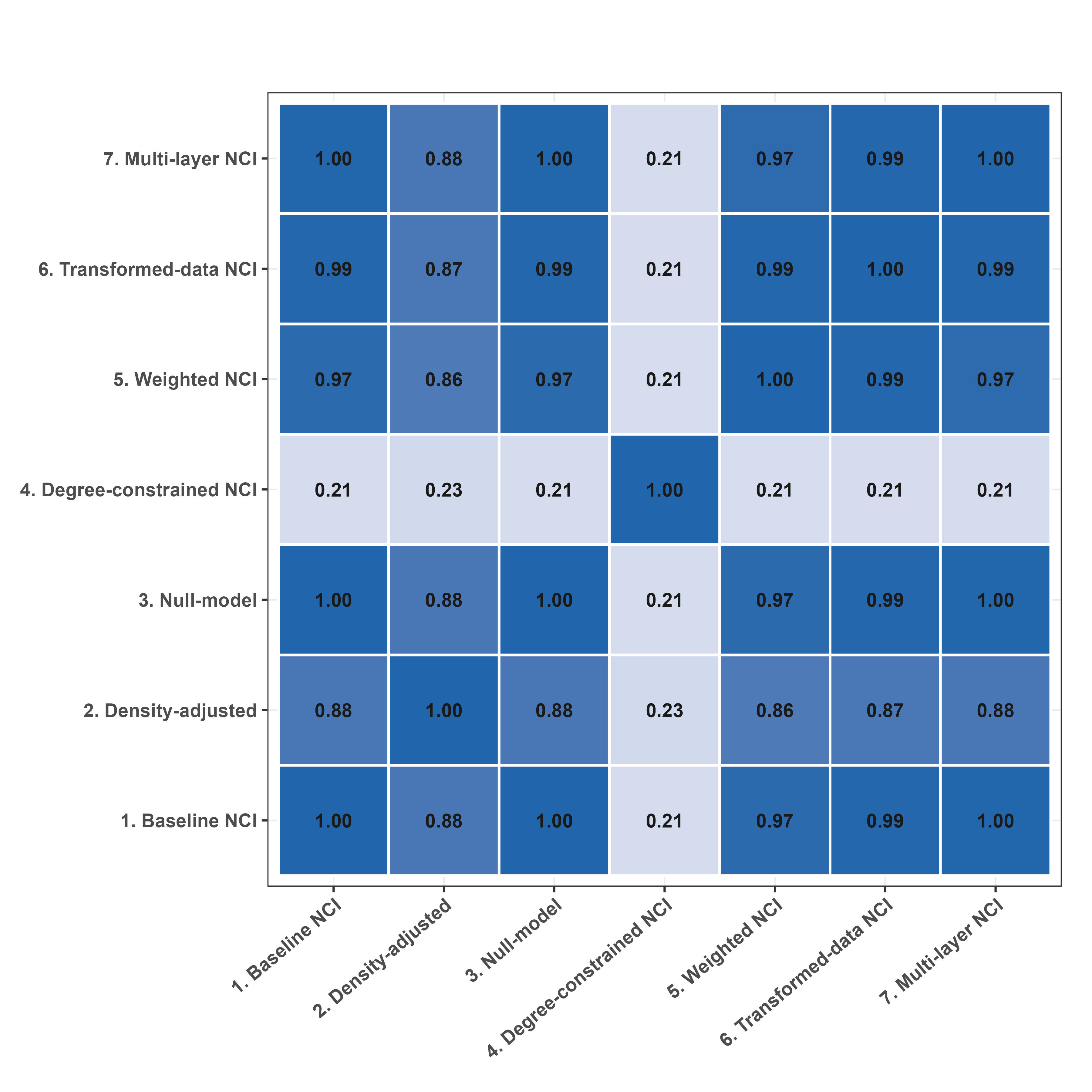}
    \subcaption{Correlation matrix of NCI variants.}
    \label{fig:corrmat}
  \end{subfigure}

  \caption{Dependence structure of network concentration measures.
  Panel~(a) shows the relationship between the seven NCI variants and
  the Herfindahl–Hirschman Index (HHI) based on the second Monte Carlo
  experiment.
  Panel~(b) reports the Pearson correlation matrix of the NCI variants,
  pooled across the three network-generating mechanisms.}
  
  \label{fig:combined_corr}
\end{figure}

\section{Empirical Applications World Input--Output Database}
\label{sec:empirical}
In the empirical applications, we focus on the baseline NCI for interpretability, while the alternative indices developed in Section \ref{sec:NCI_definition} can be applied analogously.
We consider two applications of the NCI to global economic networks constructed from the
World Input--Output Database \citep[WIOD, 2016 release;][]{timmer2015illustrated},
which provides annual input--output tables for 56 sectors across 44
countries plus a rest-of-the-world aggregate (ROW), expressed in
current prices (millions of USD). We use the 2014 cross-section
throughout.\footnote{Data are freely available at
\url{https://www.rug.nl/ggdc/valuechain/wiod/wiod-2016-release}.}
In both applications, node weights are defined as value-added shares:
\begin{equation*}
    \omega_k = \frac{\mathrm{VA}_k}{\sum_{\ell} \mathrm{VA}_\ell},
    \qquad \sum_k \omega_k = 1
\end{equation*}
where $k$ indexes either a sector or a country depending on the
application. The adjacency matrix is constructed from symmetrised
technical coefficients. Let $a_{hk} = Z_{hk}/x_k$ denote the directed
coefficient, where $Z_{hk}$ is the relevant intermediate flow and
$x_k$ gross output. Since $a_{hk} \neq a_{kh}$ in general, we
symmetrise as:
\begin{equation}
    \tilde{a}_{hk}
    = \frac{a_{hk} + a_{kh}}{2},
    \qquad \tilde{a}_{hk} = \tilde{a}_{kh}
    \label{eq:sym_coeff}
\end{equation}
and set $A_{hk} = \mathbf{1}[\tilde{a}_{hk} > \theta]$, with
$A_{hk} = A_{kh}$ and $A_{hh} = 0$.\footnote{Averaging is preferred
to the union rule ($A_{hk}=1$ if $a_{hk}>\theta$ or
$a_{kh}>\theta$), which over-connects, and to the intersection rule,
which over-prunes. \citet{molnar2023threshold} show that averaging
delivers the most stable network topology across threshold levels.}
In both applications we compare the NCI with HHI and Gini, which
depend solely on $\boldsymbol{\omega}$ and are therefore invariant to
$\theta$, serving as natural benchmarks.

\subsection{Production network}
\label{sec:production_network}

The first application constructs a sector-level production network.
Intermediate flows are aggregated over all country pairs,
$Z^{\text{agg}}_{ij} = \sum_{c,d} Z^{cd}_{ij}$, and value added
is summed across countries, $\mathrm{VA}_j = \sum_c \mathrm{VA}^c_j$.
The baseline threshold is $\theta = 0.01$, consistent with the
sector-level input--output literature \citep{acemoglu2012network,
carvalho2014micro}.

Figure~\ref{fig:sector_weights} shows the distribution of $\omega_j$.
The top four sectors --- Construction, Public Administration, Real
Estate, and Wholesale --- jointly account for approximately 20\% of
world value added, while primary activities (Forestry, Fishing) and
ancillary services carry negligible weight. The distribution is
markedly right-skewed ($\text{Gini}= 0.471$, $\mathrm{HHI} = 0.032$ as reported in Table~\ref{tab:production_results}).
From a supply chain risk perspective, the NCI value 
of $0.0777$ implies that approximately $7.8\%$ of the 
maximum achievable structural concentration is realised 
in the WIOD production network. A planner seeking to 
reduce systemic vulnerability could target a lower NCI 
by redistributing value added toward sectors that are 
less central in the input--output graph — for instance, 
by promoting activity in isolated service sectors 
(Repair, R\&D, Education) whose expansion would increase 
the denominator $1 - \mathrm{HHI}$ without 
proportionally increasing the numerator $w^\top A w$.

Figure~\ref{fig:production_network} displays the network. A dense
core of Manufacturing, Mining, Energy, and Construction sectors
occupies the centre, surrounded by a Service and Transport periphery.
Five sectors (Repair, R\&D, Veterinary, Education, Arts) are isolated
at $\theta = 0.01$, indicating that no bilateral relationship reaches
the threshold. The concentration indices are:

\begin{table}[H]\par\medskip
\scalebox{0.85}{
\centering
\caption{Concentration measures, production network
(WIOD 2014, $\theta = 0.01$).}
\label{tab:production_results}
\begin{tabular}{lc}\hline
Measure & Value \\\hline
HHI  & 0.0321 \\
Gini & 0.4710 \\
NCI  & 0.0777 \\\hline
\end{tabular}
}
\end{table}

The NCI of $0.0777$ is $2.4$ times the HHI of $0.0321$, signalling
that the most economically dominant sectors --- Construction, Real
Estate, Banking, Mining --- are also the most densely interconnected.
This concentration amplification cannot be detected by HHI or Gini
alone, as both are blind to network topology.

Figure~\ref{fig:nci_theta} plots the three indices against $\theta$
on a logarithmic scale. HHI and Gini are flat by construction; the
NCI decreases monotonically, steeply for $\theta \lesssim 0.005$
where many weak links are shed, and flattening near zero for
$\theta \gtrsim 0.03$ where only the strongest supply relationships
survive. Crucially, $\mathrm{NCI} > \mathrm{HHI}$ holds throughout
$\theta \leq 0.03$, confirming that the amplification result is not
an artefact of the chosen threshold. The crossover point between the
NCI curve and the HHI benchmark at $\theta \approx 0.013$ provides a
natural upper bound on informative threshold values: beyond this
point the network is so sparse that topology adds no concentration
signal.

\begin{figure}[H]\par\medskip
    \centering
    \includegraphics[width=0.6\linewidth]{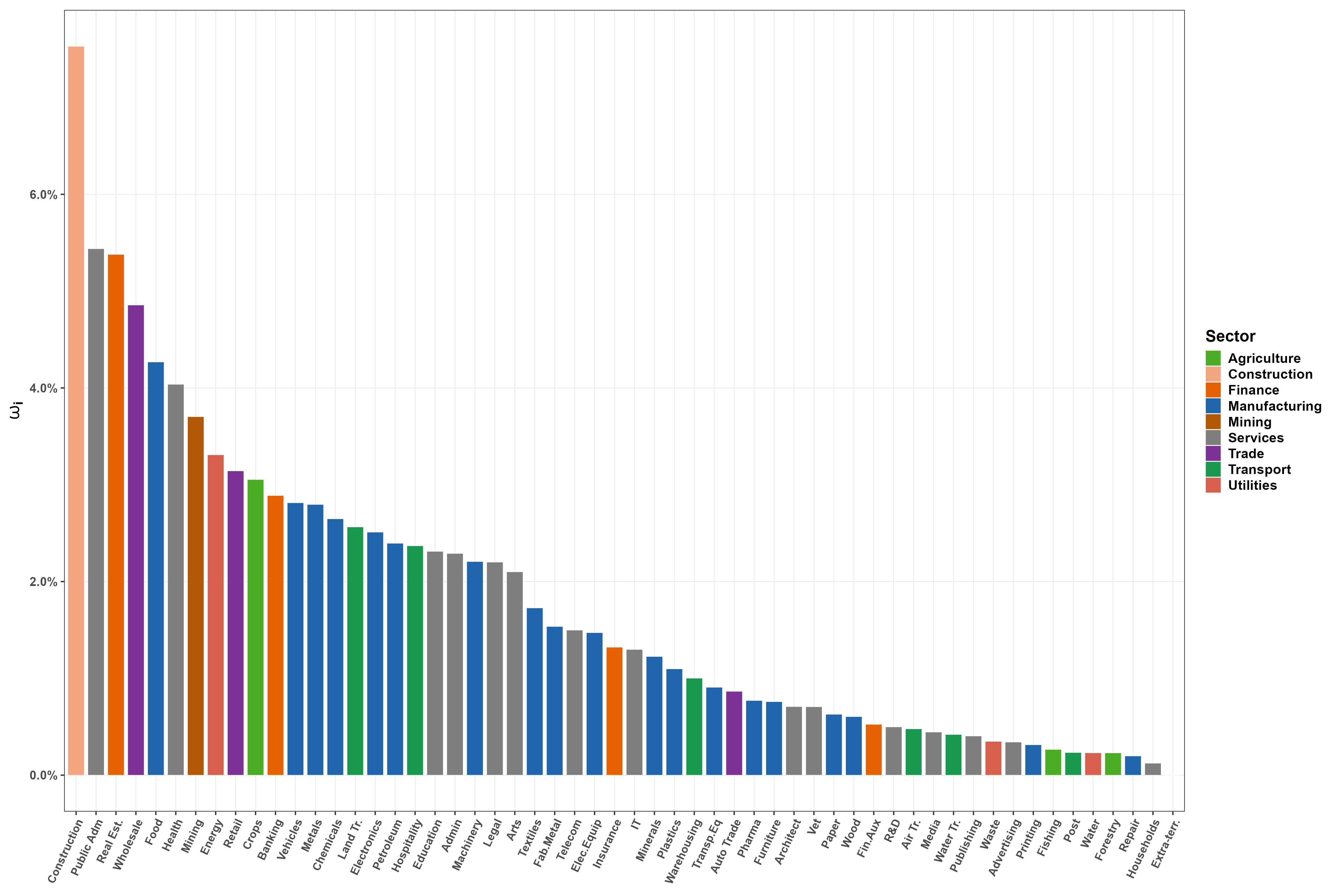}
    \caption{Sector weights $\omega_j$, ranked by decreasing value.
             Colours denote macro-sector.}
    \label{fig:sector_weights}
\end{figure}

\begin{figure}[H]\par\medskip
    \centering
    \includegraphics[width=0.9\linewidth]{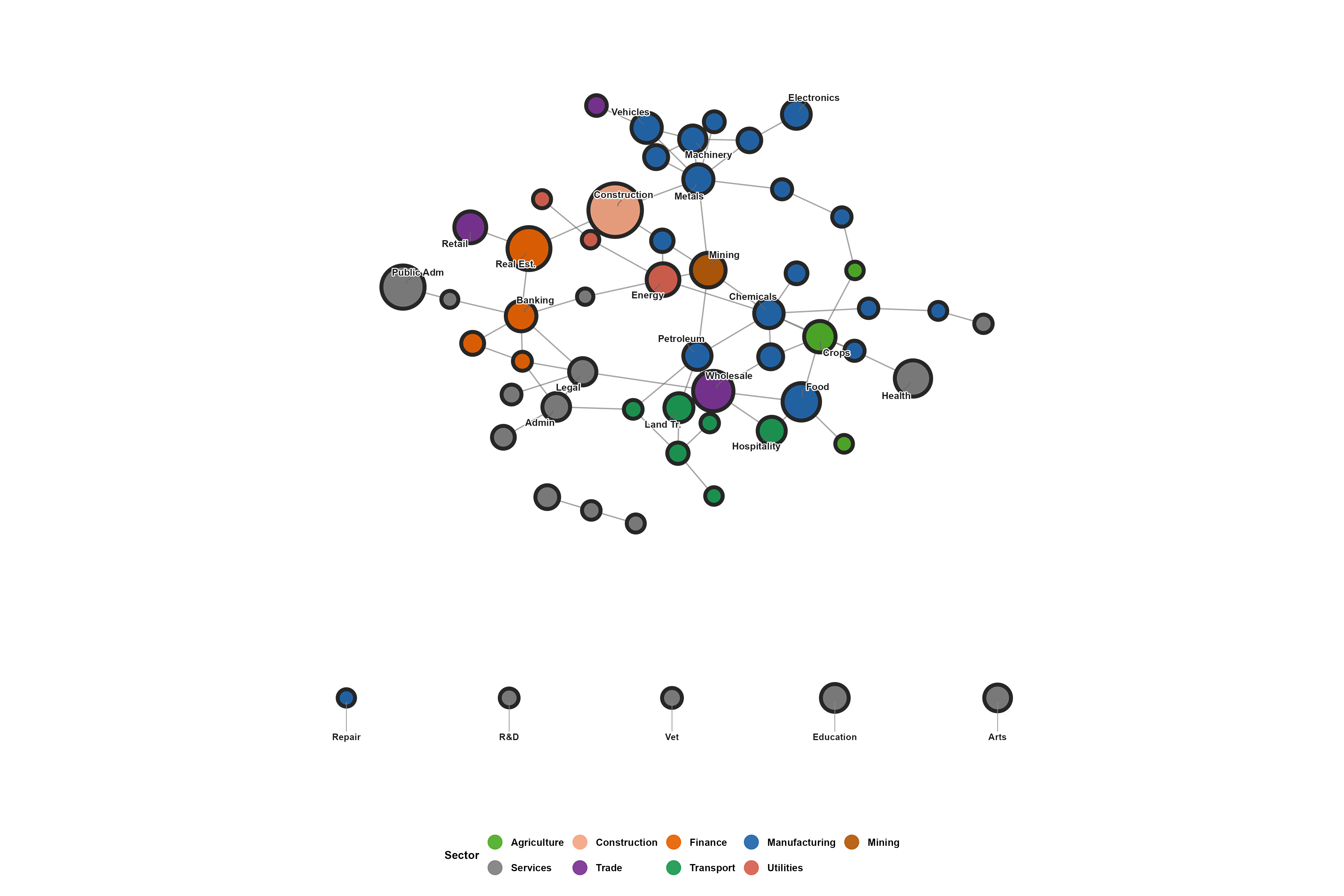}
    \caption{Symmetric production network, WIOD 2014
             ($\theta = 0.01$). }
    \label{fig:production_network}
\end{figure}

\begin{figure}[H]\par\medskip
    \centering
    \includegraphics[width=0.6\linewidth]{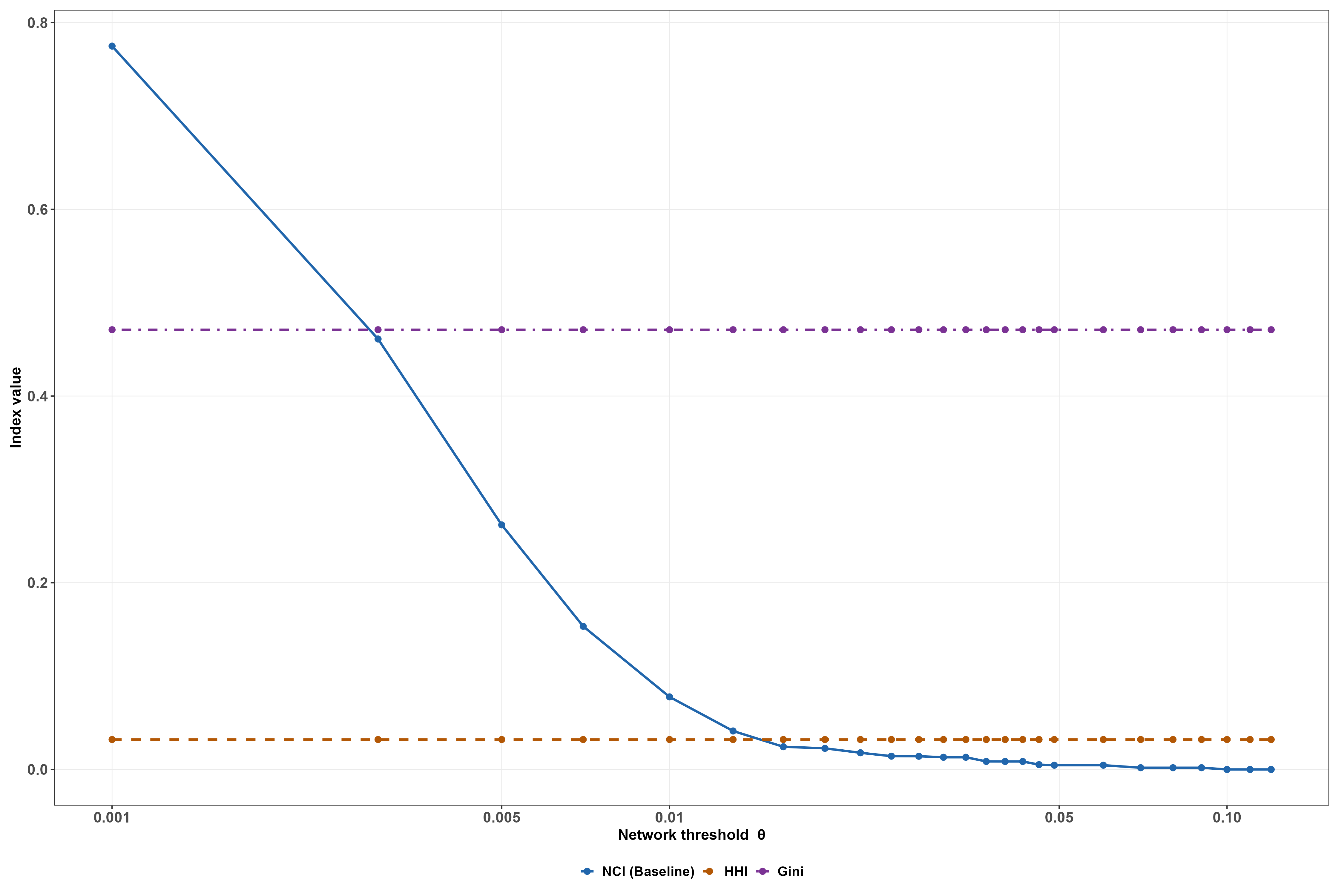}
    \caption{NCI, HHI, and Gini as functions of $\theta$
             (log scale). Production network, WIOD 2014.}
    \label{fig:nci_theta}
\end{figure}

\subsection{International trade network}
\label{sec:trade_network}

The second application constructs a country-level trade network.
Intermediate flows are now aggregated over all sector pairs,
$Z^{\text{trade}}_{cd} = \sum_{i,j} Z^{cd}_{ij}$, with diagonal
entries set to zero to exclude domestic flows ($Z^{\text{trade}}_{cc}
= 0$). Value added is summed across sectors, $\mathrm{VA}_c =
\sum_i \mathrm{VA}^c_i$. The baseline threshold is $\theta = 0.005$,
lower than in the production network application. This reflects the
greater dispersion of bilateral coefficients at the country level:
aggregating over 56 sectors spreads intermediate flows across a
larger number of partners, compressing the distribution of
$\tilde{a}_{cd}$ toward zero and requiring a lower cutoff to retain
economically meaningful links.

Figure~\ref{fig:country_weights} shows the distribution of $\omega_c$.
China and the United States jointly account for approximately 39\% of
world value added (19.71\% and 19.23\%, respectively), followed by
the Rest of World (ROW) aggregate (15.90\%). The distribution is far more
concentrated than at the sector level: $\text{Gini} = 0.726$ versus $0.471$,
and $\mathrm{HHI} = 0.113$ versus $0.032$ (see Tables~\ref{tab:trade_results} and~\ref{tab:production_results} respectively). This reflects the
well-documented asymmetry of country sizes in the global economy,
which is substantially more pronounced than the asymmetry across
industries.

Figure~\ref{fig:trade_network} reveals a markedly different topology
from the production network. Rather than a dense interconnected core,
the trade network exhibits a pronounced star structure centred on
three hubs --- China (CHN), the United States (USA), and ROW ---
which concentrate the vast majority of above-threshold bilateral
links. Western European economies (Germany, France, UK, Italy, Netherlands)
form a secondary cluster, while most smaller economies connect to the
system exclusively through one of the three dominant hubs. This
hub-and-spoke architecture has a direct implication for systemic
risk: a shock to any of the three central nodes would propagate
immediately to virtually all other countries in the network, with no
alternative routing available for peripheral economies.

\begin{table}[H]\par\medskip
\scalebox{0.85}{
\centering
\caption{Concentration measures, international trade network
(WIOD 2014, $\theta = 0.005$).}
\label{tab:trade_results}
\begin{tabular}{lc}\hline
Measure & Value \\\hline
HHI  & 0.1126 \\
Gini & 0.7256 \\
NCI  & 0.2647 \\\hline
\end{tabular}
}
\end{table}

The NCI of $0.2647$ is again approximately $2.4$ times 
the HHI of $0.1126$, consistent with the amplification 
ratio observed in the production network application. 
While this numerical similarity may reflect common 
features of WIOD-based weighting schemes rather than 
a structural invariant, it confirms that the NCI 
systematically detects a dimension of concentration 
that weight-based measures alone cannot capture, 
regardless of the underlying network topology.

Figure~\ref{fig:nci_theta_trade} shows the threshold sensitivity.
The NCI curve is considerably steeper than in the production network
case: it starts above $0.9$ at $\theta = 0.0005$, drops sharply
through the HHI benchmark at $\theta \approx 0.01$, and collapses
toward zero for $\theta \gtrsim 0.015$. This rapid collapse reflects
the star topology: once the threshold exceeds the bilateral
coefficients of hub-to-periphery links, the network becomes
essentially disconnected and the NCI loses its signal. The narrow
window $\theta \in [0.001, 0.010]$ over which
$\mathrm{NCI} > \mathrm{HHI}$ is therefore precisely the range in
which the hub structure is captured by the adjacency matrix --- a
finding that corroborates the choice of $\theta = 0.005$ as the
informative baseline.

\begin{figure}[H]\par\medskip
    \centering
    \includegraphics[width=0.6\linewidth]{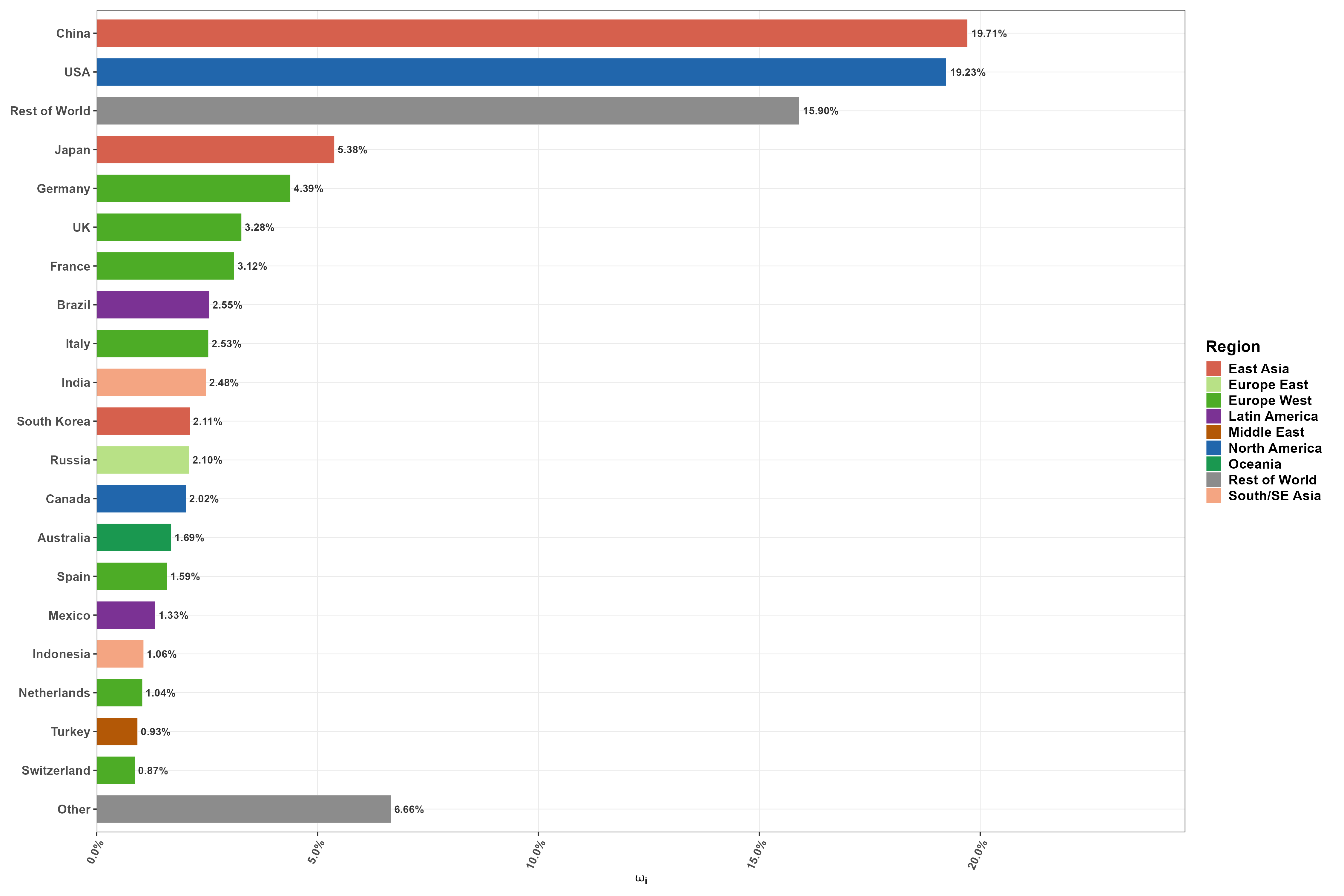}
    \caption{Country weights $\omega_c$, top-20 economies plus
             aggregate of remaining countries (Other), ranked
             by decreasing value.}
    \label{fig:country_weights}
\end{figure}

\begin{figure}[H]\par\medskip
    \centering
    \includegraphics[width=0.9\linewidth]{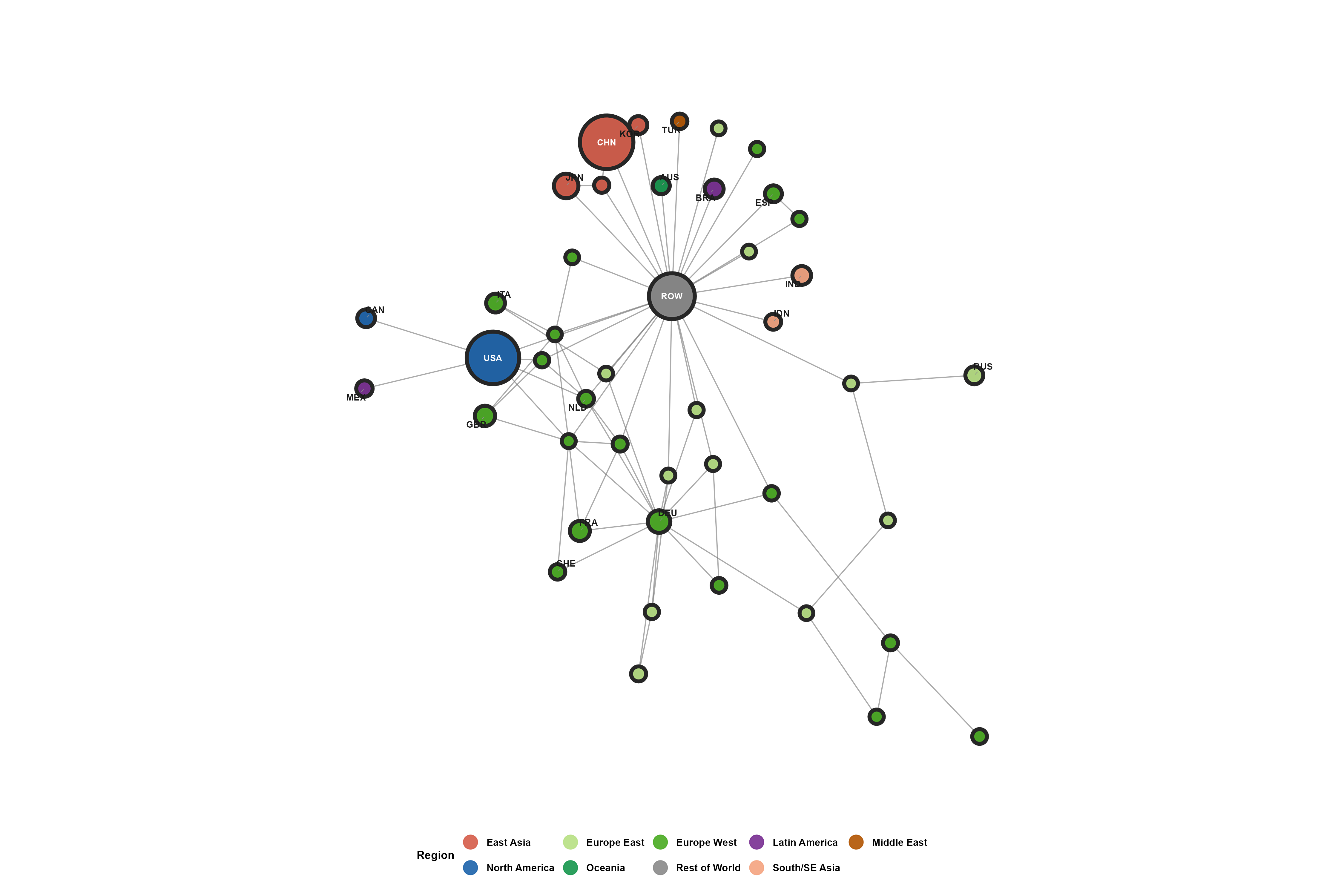}
    \caption{Symmetric international trade network, WIOD 2014
             ($\theta = 0.005$). }
    \label{fig:trade_network}
\end{figure}

\begin{figure}[H]\par\medskip
    \centering
    \includegraphics[width=0.6\linewidth]{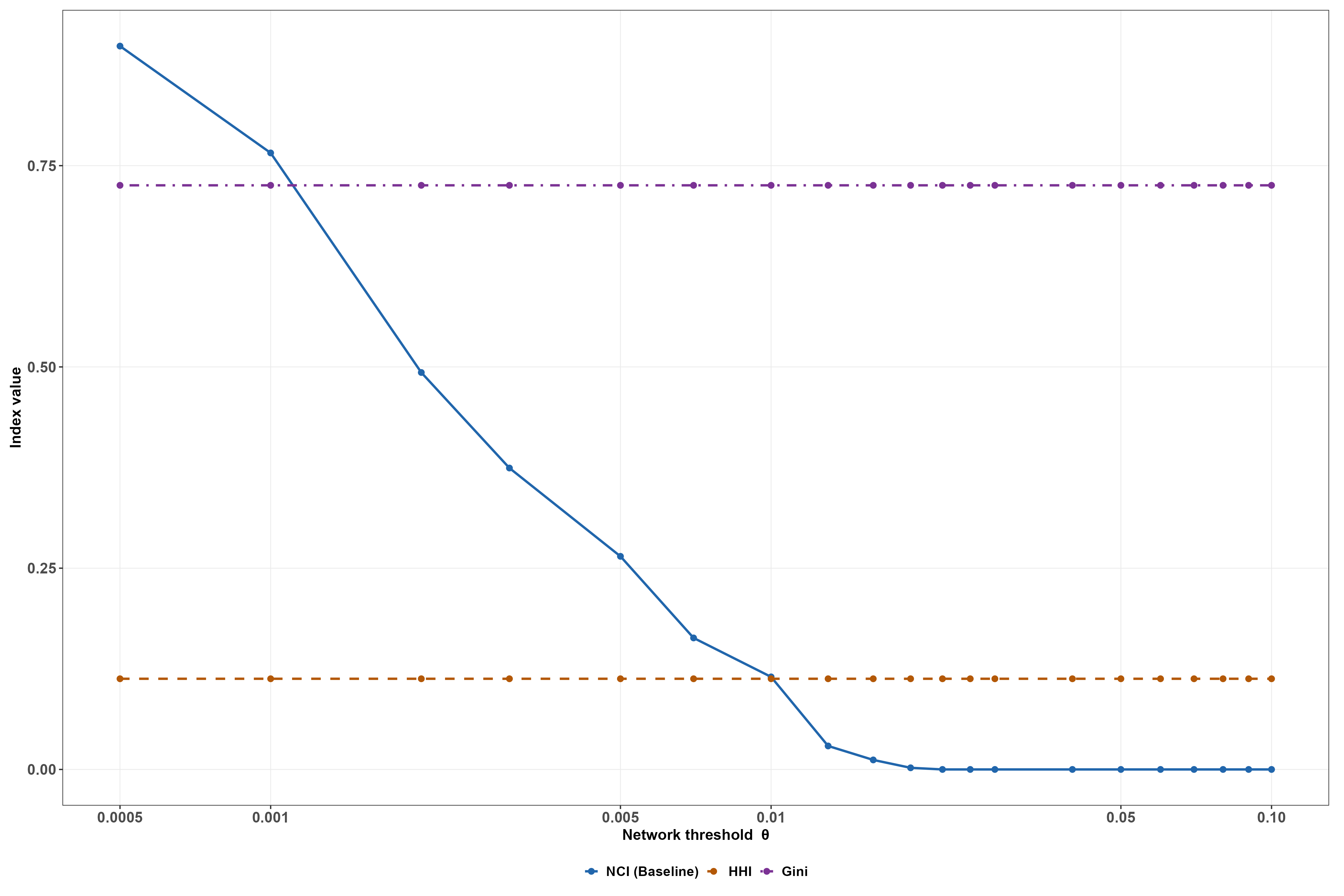}
    \caption{NCI, HHI, and Gini as functions of $\theta$
             (log scale). International trade network, WIOD 2014.}
    \label{fig:nci_theta_trade}
\end{figure}
\section{Empirical Application: Equity Dependence Network}
\label{sec:empirical_finance}

This section illustrates the empirical relevance of the NCI in a
financial market context. We construct a dependence network for the
twenty largest constituents of the S\&P~500 index by market
capitalisation, using daily adjusted closing prices downloaded from
Yahoo Finance via the \texttt{quantmod} \citep{ryan2020package} package in \textsf{R}. The
sample covers the period from 1 January 2015 to 31 December 2025,
yielding approximately 2{,}764 daily log-return observations per
asset after alignment and removal of missing values\footnote{Data
are freely available through Yahoo Finance
(\href{https://finance.yahoo.com}{ Yahoo Finance}). Adjusted prices account for
dividends and stock splits.}. The twenty stocks --- Apple (AAPL),
Microsoft (MSFT), NVIDIA (NVDA), Amazon (AMZN), Alphabet (GOOGL),
Meta (META), Berkshire Hathaway (BRK-B), Eli Lilly (LLY), Broadcom
(AVGO), Tesla (TSLA), Walmart (WMT), JPMorgan Chase (JPM), Visa (V),
ExxonMobil (XOM), UnitedHealth (UNH), Oracle (ORCL), Mastercard
(MA), Costco (COST), Home Depot (HD), and Procter \& Gamble (PG) ---
span eight GICS sectors: Technology, Communication Services, Consumer
Discretionary, Consumer Staples, Financials, Health Care, Energy, and
Industrials.
Node weights are defined as market capitalisation shares \citep{hua2019centrality}:
\begin{equation}
    \omega_i = \frac{\mathrm{MC}_i}{\sum_{j=1}^{N} \mathrm{MC}_j},
    \qquad \sum_{i=1}^{N} \omega_i = 1,
    \label{eq:mktcap_weights}
\end{equation}
where $\mathrm{MC}_i$ denotes the approximate market capitalisation of
firm $i$ as of end-2024 (in billions of USD) and $N = 20$. This
weighting scheme mirrors standard practice in equity index
construction and assigns greater importance to firms that dominate
the investable universe. Under Eq.~\eqref{eq:mktcap_weights}, the three
largest firms --- Apple (\$3{,}500bn), NVIDIA (\$3{,}300bn), and
Microsoft (\$3{,}100bn) --- collectively account for approximately
40\% of total weight, while the ten smallest firms together represent
less than 20\%. The resulting distribution is markedly right-skewed
($G = 0.4320$, $\mathrm{HHI} = 0.0879$), as illustrated in
Figure~\ref{fig:equity_weights}.

The adjacency matrix is constructed from the Minimum Spanning Tree
(MST) \citep{kruskal1956shortest,mantegna1999hierarchical,chow1968approximating}[see, among others] of the pairwise correlation structure of log-returns.
Let $r_{it} = \log(P_{it}/P_{i,t-1})$ denote the log-return of
asset $i$ at time $t$, and let $\rho_{ij}$ be the sample Pearson
correlation between $r_{i}$ and $r_{j}$ over the full sample.
Following \citet{mantegna1999hierarchical}, we convert correlations
into an ultrametric distance:
\begin{equation}
    d_{ij} = \sqrt{2\,(1 - \rho_{ij})},
    \qquad d_{ij} \in [0,\, 2],
    \label{eq:mantegna_dist}
\end{equation}
which satisfies $d_{ij} = 0$ if and only if $\rho_{ij} = 1$ and
$d_{ij} = 2$ if and only if $\rho_{ij} = -1$. We then extract the
MST of the complete weighted graph $(\mathcal{V}, \mathcal{E},
\mathbf{D})$ using Prim's algorithm, retaining the $N - 1 = 19$
edges that minimise total distance. The binary adjacency matrix is
defined as $A_{ij} = 1$ if edge $(i,j)$ belongs to the MST, and
$A_{ij} = 0$ otherwise, with $A_{ii} = 0$ throughout.

The MST approach differs fundamentally from the threshold-based
adjacency matrix used in the WIOD applications
(Section~\ref{sec:empirical}). Rather than applying an arbitrary
cutoff $\theta$, it extracts a parsimonious backbone of the
correlation structure that is invariant to monotone transformations
of the distance metric, avoids the over-connectivity problem
associated with dense correlation matrices, and yields a unique
connected graph for any set of distinct pairwise correlations.
This makes it particularly well suited to equity markets, where
nearly all pairs of stocks exhibit positive correlations and
threshold-based filtering would either produce a trivially complete
graph (low $\theta$) or an overly sparse, disconnected one (high
$\theta$).

Figure~\ref{fig:equity_weights} displays the distribution of
$\omega_i$ ranked by decreasing value and coloured by GICS sector.
Technology firms dominate: the four largest weights belong to Apple
(15.1\%), NVIDIA (14.2\%), and Microsoft (13.4\%), which together
account for approximately 43\% of total market capitalisation in the
sample. Communication Services (Alphabet, Meta) and Consumer
Discretionary (Amazon, Tesla) follow at a considerable distance.
The remaining twelve firms, drawn from Financials, Health Care,
Consumer Staples, and Energy, each contribute less than 5\% of total
weight. The distribution is considerably less concentrated than the
country-level distribution in the international trade application
($G = 0.726$, $\mathrm{HHI} = 0.113$, as reported in Table~\ref{tab:equity_results}), reflecting the more homogeneous
size distribution within the S\&P~500 mega-cap segment relative to
the global distribution of national income.

Figure~\ref{fig:equity_network} displays the MST-based dependence
network. Two hubs emerge prominently: Microsoft (MSFT), with five
direct connections, and Berkshire Hathaway (BRK-B), with four.
Microsoft sits at the centre of the Technology--Communication cluster,
linking Apple, NVIDIA, Alphabet, Amazon, Oracle, and Mastercard,
which reflects the co-movement of large-cap growth stocks driven by
common exposure to interest rates, earnings revisions, and AI-related
sentiment. Berkshire Hathaway anchors a second cluster of
non-technology value stocks --- JPMorgan, Visa, ExxonMobil,
UnitedHealth, and Home Depot --- consistent with its role as a
diversified holding company with significant exposure to the
financial, energy, and consumer cyclical sectors. The Consumer
Staples stocks (Walmart, Costco, Procter \& Gamble) form a peripheral
chain, connected to the rest of the network through a single path,
indicating lower co-movement with the core Technology and Financials
nodes. No stocks are isolated in the MST by construction, since
the algorithm always produces a connected spanning tree. The
concentration indices are reported in Table~\ref{tab:equity_results}.

\begin{table}[H]\par\medskip
\scalebox{0.85}{
\centering
\caption{Concentration measures, equity dependence network
(S\&P~500 top-20 by market cap, 2015--2024).}
\label{tab:equity_results}
\begin{tabular}{lc}\hline
Measure & Value \\\hline
HHI & 0.0879 \\
Gini                   & 0.4320 \\
NCI & 0.1939 \\\hline
\end{tabular}
}
\end{table}

The NCI of $0.1939$ is approximately $2.2$ times the HHI of $0.0879$,
a concentration amplification ratio broadly consistent with those
observed in the WIOD applications ($2.4\times$ for both the
production and trade networks). This amplification signals that
the most economically dominant firms --- Apple, NVIDIA, Microsoft ---
are simultaneously the most central in the return co-movement
network, so that market-cap concentration and network centrality
reinforce each other. Neither the HHI nor the Gini captures this
reinforcement, as both are blind to network topology: the HHI
reflects only the squared weight of individual firms, while the
Gini measures only the dispersion of the weight distribution.
The NCI of $0.1939$ implies that approximately 19.4\% of the
maximum achievable network concentration is realised in this
portfolio, a figure substantially higher than the 8.8\% suggested
by the HHI benchmark. From a portfolio risk perspective, this
gap is economically meaningful: an investor relying solely on the
HHI as a diversification gauge would underestimate by more than
half the extent to which concentration risk is amplified by the
correlation structure of returns. Figure~\ref{fig:rolling_nci} plots the three indices computed on
rolling windows of 252 trading days (approximately one calendar year)
with a step of 63 days (one quarter), over the period 2015--2025.
The shaded band around the NCI series reports the pointwise 95\%
bootstrap confidence interval \citep{efron1992bootstrap} obtained by resampling daily returns
with replacement within each window ($B = 500$ resamples), so that
any variation in the NCI across windows reflects genuine shifts in
the MST topology \citep{mantegna1999hierarchical} rather than
estimation noise. By construction, HHI and Gini are constant across
windows, since the market capitalisation weights $\omega_i$ are fixed
at their end-2024 values and do not vary with the estimation window;
they therefore carry no confidence band.

Four features of the NCI time series are noteworthy. First, the NCI
is persistently above the HHI throughout the sample and the 95\%
confidence band never overlaps the HHI line, confirming that the
concentration amplification result is statistically significant and
not specific to any particular sub-period or market regime. Second,
the NCI exhibits a clear upward trend from approximately 0.10 in
mid-2015 to a range of 0.16--0.21 by 2017--2020, which we attribute
to the progressive concentration of returns around a small number of
mega-cap Technology stocks as their share of the S\&P~500 index
grew over this period. Third, the NCI peaks at approximately $0.21$
in the period surrounding the COVID-19 shock of early 2020, when
return correlations surged across virtually all asset classes and
the MST contracted toward a star topology centred on the largest
Technology hubs. A one-sided Welch $t$-test confirms that the mean
NCI during the COVID-19 period (2020) is significantly higher than
in the pre-COVID baseline ($p < 0.01$), whereas the HHI remains
unchanged by construction --- a result consistent with the
flight-to-quality and momentum dynamics documented in the literature
\citep{onnela2003dynamic, tumminello2005tool}. Fourth, the NCI
subsequently stabilises in the $0.16$--$0.20$ range through
2021--2025, with a modest decline in the most recent quarters
possibly reflecting sector rotation away from Technology into
Financials, Energy, and Health Care; the narrowing of the confidence
band in this sub-period suggests that the correlation structure
became more stable relative to the high-volatility episodes of
2018--2020. These dynamics are entirely invisible to the HHI and
Gini, which remain flat by construction, underscoring the
informational content of the network dimension of the NCI.

\begin{figure}[H]\par\medskip
    \centering
    \includegraphics[width=0.5\linewidth]{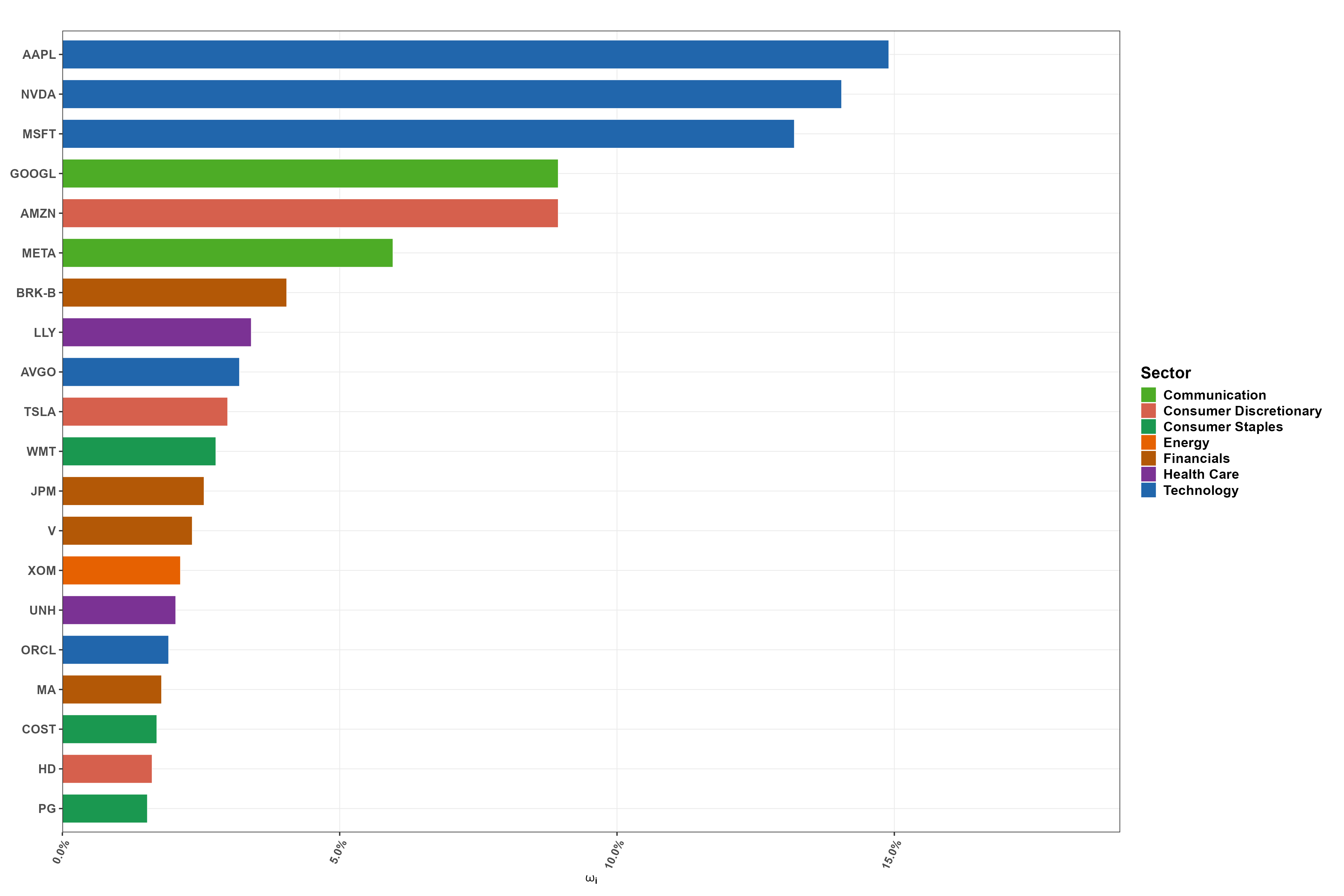}
    \caption{Market capitalisation weights $\omega_i$, S\&P~500
             top-20 stocks ranked by decreasing value.
             Colours denote GICS sector. End-2024 market caps
             (approximate, billions USD).}
    \label{fig:equity_weights}
\end{figure}

\begin{figure}[H]\par\medskip
    \centering
    \includegraphics[width=0.9\linewidth]{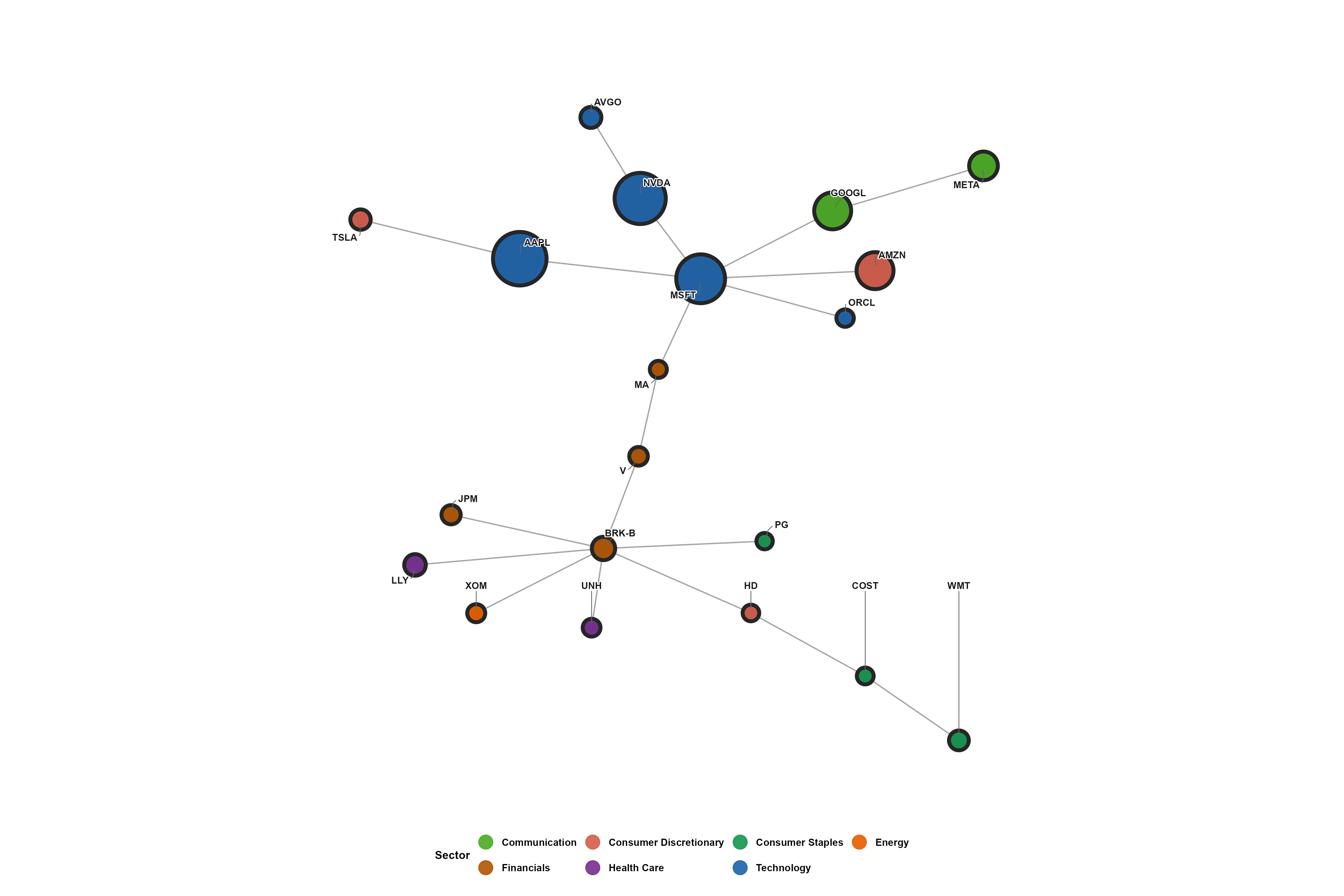}
    \caption{MST-based equity dependence network, S\&P~500 top-20,
             2015--2024.}
    \label{fig:equity_network}
\end{figure}
\begin{figure}[H]\par\medskip
    \centering
    \includegraphics[width=0.6\linewidth]{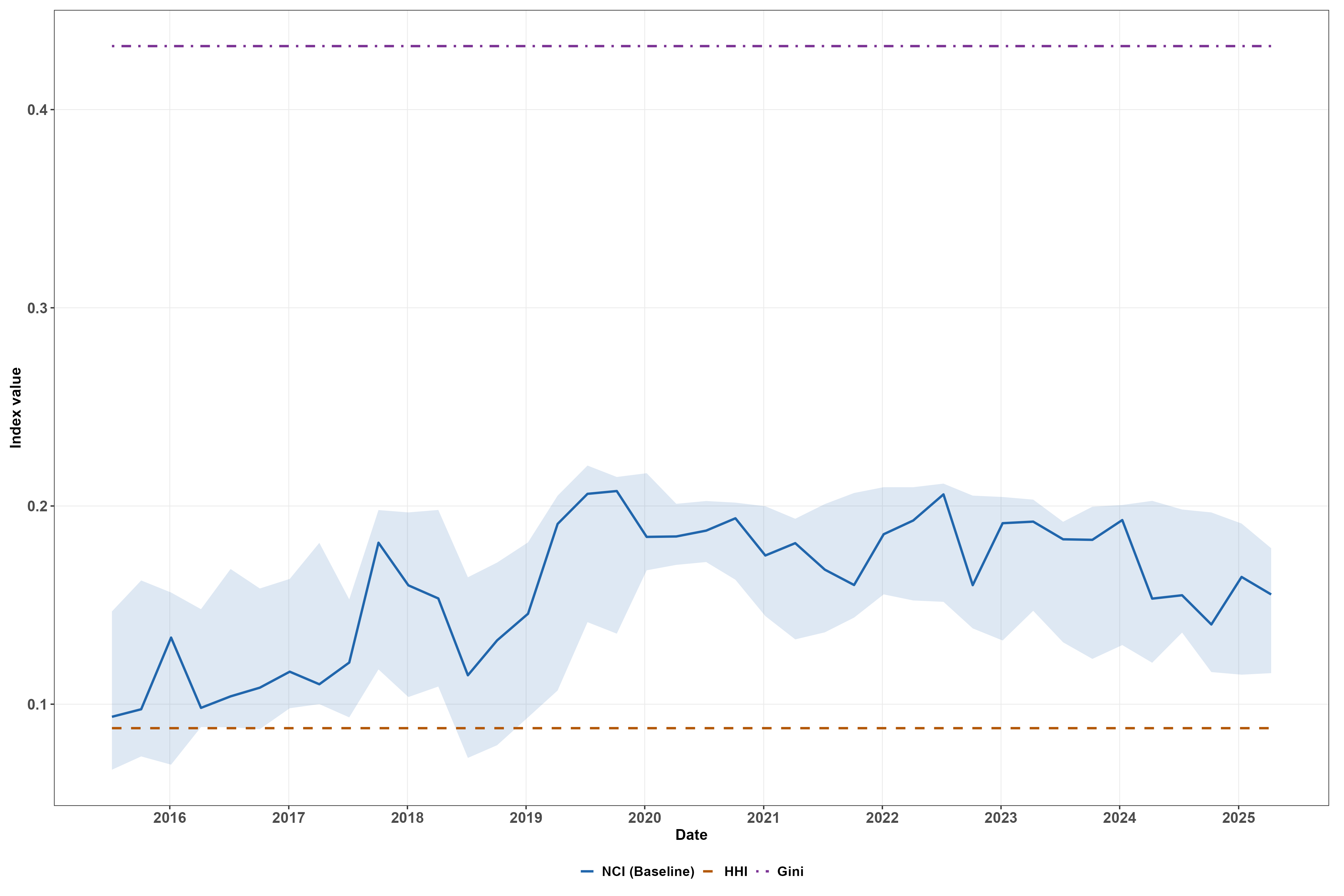}
    \caption{NCI, HHI, and Gini on rolling 252-day windows (step: 63 days),
S\&P~500 top-20, 2015--2025. The shaded band reports the
pointwise 95\% bootstrap confidence interval for the NCI
($B = 500$ resamples per window).}
    \label{fig:rolling_nci}
\end{figure}

\section{Conclusions}
\label{sec:conclusions}
This paper develops a unified framework for measuring concentration in weighted networks, within which the Network Concentration Index (NCI) arises as a baseline specification.
Traditional concentration measures such as the
Herfindahl--Hirschman Index depend solely on the dispersion of
weights and therefore ignore how economically important nodes are
positioned within the network. The concentration indices address this limitation
by combining node weights with the topology of the adjacency
structure, thereby providing a structural measure of concentration.

Within this framework, several normalization schemes are proposed in order to interpret
 the index relative to different benchmarks, including density
 adjustments, stochastic null models, and degree-constrained
 configurations.

Each normalization isolates a different dimension
of network organization. In particular, the degree-constrained
variant measures the alignment between node weights and network
connectivity conditional on the observed degree sequence,
capturing a structural dimension that is largely orthogonal to
the other variants.

Monte Carlo simulations confirm the theoretical properties of the
index and illustrate its behaviour under alternative network
structures. The results show that the NCI responds not only to the
dispersion of node weights but also to the topology of the
network. In particular, the index increases when highly weighted
nodes are strongly interconnected and decreases when weight is
concentrated on peripheral nodes. The simulations also reveal that
most variants convey similar information, while the
degree-constrained index captures a distinct structural dimension.

The empirical applications further demonstrate the usefulness of
the index in practice. In production networks, international trade
networks, and financial dependence networks, the NCI consistently
exceeds traditional concentration measures, indicating that 
structural concentration arises not only from unequal weights but
also from the central position of large nodes within the network.
These results highlight how network topology can amplify or
attenuate the effective concentration of economic activity.

Overall, the proposed framework provides a flexible approach for studying concentration in networked systems, with the NCI serving as a baseline measure within a broader family of topology-aware indices.
By separating the effects of
weight dispersion from the architecture of connections, the index
offers a richer perspective on how economic importance is
distributed and transmitted across networks. Future work may
extend this framework to dynamic networks, weighted adjacency
structures, and applications in systemic risk, industrial
organization, and international production networks.
\bibliographystyle{elsarticle-num-names}
\bibliography{reference}
\appendix

\section{Appendix: Properties of the family of Network Concentration Indices}
\label{sec:appendix}       
This appendix collects the main analytical properties of the family of
Network Concentration Indices introduced in the paper. We do not repeat
the definitions of the individual indices, but instead derive first the
properties that are shared by the unified class and then the additional
results that characterize specific variants.
\subsection{General properties of the unified family}
\label{sec:appendix_general}  

Consider the unified formulation
\begin{gather}
\Psi(w;M,B)=\frac{w^\top M w}{w^\top B w},
\qquad w^\top B w>0
\end{gather}
where $w=(w_1,\dots,w_N)^\top$ is a non-negative weight vector such that
$\sum_{i=1}^N w_i=1$, and $M,B$ are symmetric matrices with zero diagonal
and non-negative off-diagonal entries.

\begin{proposition}[Weighted-average representation]
\label{app:prop:wa}
For any admissible pair $(M,B)$ with $B=11^\top-I$,
\begin{gather}
\Psi(w;M,11^\top-I)
=
\frac{\sum_{i\neq j} w_i w_j M_{ij}}
     {\sum_{i\neq j} w_i w_j}
\end{gather}
Hence $\Psi$ is a weighted average of the pairwise interaction
intensities $M_{ij}$, with weights proportional to $w_i w_j$.
\end{proposition}

\begin{proof}
Since $M_{ii}=0$,
\begin{gather}
w^\top M w=\sum_{i\neq j} w_i w_j M_{ij}
\end{gather}
Moreover,
\begin{gather}
w^\top (11^\top-I) w
= 1-\sum_i w_i^2
= \sum_{i\neq j} w_i w_j
\end{gather}
Substituting into $\Psi$ yields the result.
\end{proof}

\begin{proposition}[Permutation invariance]
\label{app:prop:perm}
Let $P$ be a permutation matrix. Then
\begin{gather} 
\Psi(Pw;PMP^\top,PBP^\top)=\Psi(w;M,B)
\end{gather}
\end{proposition}

\begin{proof}
Using $P^\top P=I$,
\begin{gather}
(Pw)^\top (PMP^\top)(Pw)=w^\top M w,
\qquad
(Pw)^\top (PBP^\top)(Pw)=w^\top B w
\end{gather}
Taking the ratio gives the claim.
\end{proof}

\begin{proposition}[Nonnegativity]
\label{app:prop:nonneg}
Under the maintained assumptions on $M$ and $B$,
\begin{gather}
\Psi(w;M,B)\ge 0
\end{gather}
\end{proposition}

\begin{proof}
Since $w_i\ge 0$ and $M_{ij}\ge 0$ for $i\neq j$, we have $w^\top M w\ge 0$.
As $w^\top B w>0$ by assumption, it follows that $\Psi(w;M,B)\ge 0$.
\end{proof}

\begin{proposition}[Homogeneity]
\label{app:prop:hom}
For any scalar $c>0$,
\begin{gather}
\Psi(w;cM,B)=c\,\Psi(w;M,B)
\end{gather}
\end{proposition}

\begin{proof}
Immediate from
\begin{gather}
\Psi(w;cM,B)=\frac{w^\top (cM) w}{w^\top B w}
=c\,\frac{w^\top M w}{w^\top B w}.
\end{gather}
\end{proof}

\subsection{Properties of the baseline-normalized family}
\label{sec:appendix_baseline}
Now consider the family of indices with complete-network benchmark
\begin{gather}
\psi^{(M)}(w)=\frac{w^\top M w}{1-\sum_{i=1}^N w_i^2}
\end{gather}
which corresponds to $B=11^\top-I$.

\begin{proposition}[Relation to the Herfindahl--Hirschman Index]
\label{app:prop:hhi}
The denominator of the baseline-normalized family satisfies
\begin{gather}
1-\sum_{i=1}^N w_i^2 = 1-\mathrm{HHI}(w),
\end{gather}
where $\mathrm{HHI}(w)=\sum_{i=1}^N w_i^2$.
\end{proposition}

\begin{proof}
Immediate from the definition of the Herfindahl--Hirschman Index.
\end{proof}

\begin{proposition}[Bounds under binary interaction matrices]
\label{app:prop:bounds}
If $M=A$ is a binary symmetric adjacency matrix with zero diagonal, then
\begin{gather}
0\le \psi^{(A)}(w)\le 1
\end{gather}
\end{proposition}

\begin{proof}
Non-negativity follows from Proposition~\ref{app:prop:nonneg}. Since
$A_{ij}\le 1$ for all $i\neq j$,
\begin{gather}
w^\top A w=\sum_{i\neq j} w_i w_j A_{ij}
\le
\sum_{i\neq j} w_i w_j
=
1-\sum_i w_i^2
\end{gather}
Dividing by the denominator gives the upper bound.
\end{proof}

\begin{proposition}[Equal-weight benchmark]
\label{app:prop:eq}
If $w_i=1/N$ for all $i$, then for a binary adjacency matrix $A$,
\begin{gather}
\psi^{(A)}(w)=\frac{2|E|}{N(N-1)}=\delta(A)
\end{gather}
that is, the index coincides with network density.
\end{proposition}

\begin{proof}
If $w_i=1/N$ for all $i$,
\begin{gather}
w^\top A w = \sum_{i\neq j}\frac{1}{N^2}A_{ij}
=\frac{2|E|}{N^2}
\end{gather}
while
\begin{gather}
1-\sum_i w_i^2 = 1-\frac{1}{N}=\frac{N-1}{N}
\end{gather}
Hence
\begin{gather}
\psi^{(A)}(w)
=
\frac{2|E|/N^2}{(N-1)/N}
=
\frac{2|E|}{N(N-1)}
\end{gather}
\end{proof}

\begin{proposition}[Random-network benchmark]
\label{app:prop:ER}
If $A$ is an \cite{erdHos1960evolution} random graph with link probability $p$,
then
\begin{gather}
\mathbb{E}\big[\psi^{(A)}(w)\big]=p
\end{gather}
\end{proposition}

\begin{proof}
Since $\mathbb{E}[A_{ij}]=p$ for $i\neq j$,
\begin{gather}
\mathbb{E}[w^\top A w]
=
\sum_{i\neq j} w_i w_j \mathbb{E}[A_{ij}]
=
p\sum_{i\neq j} w_i w_j
=
p\Big(1-\sum_i w_i^2\Big)
\end{gather}
Dividing by the denominator yields the result.
\end{proof}
\subsection{Specific properties of individual variants}
\label{sec:appendix_variants} 

\begin{proposition}[Density-adjusted index]
\label{app:prop:dens}
For the density-adjusted index,
\begin{gather}
\psi^{(\mathrm{dens})}(w,A)
=
\frac{\psi(w,A)}{\delta(A)}
=
\frac{\mathbb{E}[w_i w_j\mid A_{ij}=1]}
       {\mathbb{E}[w_i w_j]}
\end{gather}
Hence it measures assortative connectivity with respect to node weights.
\end{proposition}

\begin{proof}
Follows from Proposition~\ref{prop:assort} by normalization with $\delta(A)$.
\end{proof}

\begin{proposition}[Weighted index]
\label{app:prop:weighted}
For the weighted index,
\begin{gather}
\psi^{(W)}(w)=\frac{w^\top W w}{1-\sum_i w_i^2},
\qquad
\psi^{(cW)}(w)=c\,\psi^{(W)}(w), \quad c>0.
\end{gather}
\end{proposition}

\begin{proof}
Immediate from Proposition~\ref{app:prop:hom}.
\end{proof}

\begin{proposition}[Multi-layer index]
\label{app:prop:multi}
Let
\begin{gather}
A^{(\alpha)}=\sum_{\ell=1}^L \alpha_\ell A^{(\ell)},
\qquad
\alpha_\ell\ge 0,\ \sum_{\ell=1}^L \alpha_\ell=1.
\end{gather}
Then
\begin{gather}
\psi^{(\alpha)}(w)
=
\sum_{\ell=1}^L \alpha_\ell \psi^{(A^{(\ell)})}(w).
\end{gather}
\end{proposition}

\begin{proof}
By linearity, $w^\top A^{(\alpha)} w=\sum_{\ell} \alpha_\ell w^\top A^{(\ell)} w$.
Dividing by the common denominator yields the result.
\end{proof}

\begin{proposition}[Degree-constrained index]
\label{app:prop:deg}
Let
\begin{gather}
\Lambda(w;\mathcal{G}(d))=\max_{B\in\mathcal{G}(d)} w^\top B w.
\end{gather}
Then
\begin{gather}
0\le \psi^{(\mathrm{deg})}(w,A)\le 1,
\end{gather}
with equality $\psi^{(\mathrm{deg})}(w,A)=1$ if and only if $A$ attains
the maximum.
\end{proposition}

\begin{proof}
By construction $w^\top A w \le \Lambda(w;\mathcal{G}(d))$, and both terms
are nonnegative, implying the bounds. Equality holds iff $A$ is a maximizer.
\end{proof}

\end{document}